%% file: main.tex
\documentclass{article}

\PassOptionsToPackage{round}{natbib}
\usepackage[preprint]{neurips_2025}

\input{packages}

\input{notations}

\title{
    Active Human Feedback Collection via Neural Contextual Dueling Bandits
}

\author{Arun Verma$^{1}$,
    Xiaoqiang Lin$^{2}$,
    Zhongxiang Dai$^{3}$,\\
    \textbf{Daniela Rus}$^{1,4}$\textbf{,} 
    \textbf{Bryan Kian Hsiang Low}$^{1,2}$\\
    $^{1}$Singapore-MIT Alliance for Research and Technology, Republic of Singapore \\
    $^{2}$Department of Computer Science, National University of Singapore, Republic of Singapore\\
    $^{3}$The Chinese University of Hong Kong, Shenzhen, China \\
    $^{4}$CSAIL, MIT, USA\\
    \texttt{arun.verma@smart.mit.edu}, \texttt{xiaoqiang.lin@u.nus.edu}, \texttt{daizhongxiang@cuhk.edu.cn}, \\
    \texttt{rus@csail.mit.edu}, \texttt{lowkh@comp.nus.edu.sg}
}

\begin{document}    
    \maketitle

    \begin{abstract}
        Collecting human preference feedback is often expensive, leading recent works to develop principled algorithms to select them more efficiently. However, these works assume that the underlying reward function is linear, an assumption that does not hold in many real-life applications, such as online recommendation and LLM alignment. To address this limitation, we propose \algo{}, an algorithm based on the neural contextual dueling bandit framework that provides a principled and practical method for collecting human preference feedback when the underlying latent reward function is non-linear. We theoretically show that when preference feedback follows the Bradley-Terry-Luce model, the worst sub-optimality gap of the policy learned by \algo{} decreases at a sub-linear rate as the preference dataset increases. Our experimental results on preference datasets further corroborate the effectiveness of \algo{}.
    \end{abstract}

    \section{Introduction}
    \label{sec:introduction}
    \input{latex/introduction}

    \section{Problem Setting}
    \label{sec:problem}
    \input{latex/problem}

    \section{Algorithm for Active Human Preference Feedback Collection}
    \label{sec:neural_adb}
    \input{latex/neural_adb}

        \subsection{Theoretical Results}
        \label{ssec:theory_results}
        \input{latex/theory_results}

        \subsection{Active Dueling Bandits with Regret Minimization}
        \label{ssec:adb_regret}
        \input{latex/adb_regret}

    \section{Experiments}
    \label{sec:experiments}
    \input{latex/experiment}

    \section{Related Work}
    \label{sec:related_work}
    \input{latex/related_work}

    \section{Conclusion}
    \label{sec:conclusion}
    \input{latex/conclusion}

    \subsection*{Acknowledgements}
    This research is supported by the National Research Foundation (NRF), Prime Minister’s Office, Singapore under its Campus for Research Excellence and Technological Enterprise (CREATE) programme. The Mens, Manus, and Machina (M3S) is an interdisciplinary research group (IRG) of the Singapore MIT Alliance for Research and Technology (SMART) centre.

    \subsection*{Impact Statement}
    This paper's primary contributions are theoretical, so we do not anticipate any immediate negative societal impacts. However, our algorithms could be used to improve online DPO/RLHF for efficient LLM alignment, which would benefit society. While the potential negative societal impacts associated with LLM alignment may also extend to our work, mitigation strategies for preventing the misuse of LLM alignment would also help safeguard against any potential misuse of our algorithms.

    \bibliographystyle{unsrtnat} 
    \bibliography{references}

    \newpage
    \onecolumn
    \appendix
    
    \section{Leftover Proofs}
    \label{sec:appendix}
    \input{latex/appendix}
    \section{Additional Experimental Details and Results}
    \label{sec:appendix_experiments}
    \input{latex/appendix_experiments}

    \hrule height 0.5mm

\end{document}

%% file: packages.tex
\usepackage[utf8]{inputenc}   
\usepackage[T1]{fontenc}    	
\usepackage{url}              		 
\usepackage{booktabs}       	
\usepackage{amsfonts}       	
\usepackage{nicefrac}       	  
\usepackage{microtype}      	

\usepackage{graphicx}
\usepackage{color}
\usepackage{rotating}
\usepackage{tabularx}
\usepackage{pdflscape}
\usepackage{amsmath,amsthm,amssymb}
\usepackage{algorithm,algorithmic}
\usepackage{bbm,dsfont}
\usepackage[caption=false]{subfig}
\usepackage{wrapfig}
\usepackage{cancel}
\usepackage{enumerate,cases}
\usepackage{thmtools,thm-restate}
\usepackage{mathtools}

\usepackage[none]{hyphenat}
\usepackage{multirow}
\usepackage{makecell}
\usepackage{setspace}
\usepackage{pifont}
\usepackage{array}
\usepackage{enumitem}
\usepackage{bm}
\usepackage{lineno}     

\allowdisplaybreaks

\usepackage{silence}
\usepackage{hyperref}		 
\hypersetup{
	colorlinks	= true,		 
	urlcolor     = blue,	 
	linkcolor	 = purple, 	 
	citecolor    = violet    
}
\usepackage[capitalize]{cleveref}

\usepackage{todonotes}

%% file: notations.tex
\newcommand{\para}[1]{\textbf{#1}~~}
\newcommand{\algo}{\textnormal{\texttt{Neural-ADB}}}
\newcommand{\adb}{\textsc{ADB}}

\newcommand{\cf}{\text{cf}}

\newcommand{\Regret}{\kR}

\newcommand{\EE}[1]{\bE\left[#1\right]}

\newcommand{\Prob}[1]{\bP\left\{#1\right\}}

\newcommand{\R}{\bR}

\newcommand{\one}[1]{\mathds{1}_{\left\{#1\right\}}}
\newcommand{\bsym}{\boldsymbol}

\renewcommand{\phi}{\varphi}
\renewcommand{\epsilon}{\varepsilon}

\newcommand{\norm}[1]{\left\|#1\right\|}

\newcommand{\al}[1]{ \begin{align} #1  \end{align}}
\newcommand{\eq}[1]{ \begin{equation} #1  \end{equation}}
\newcommand{\als}[1]{ \begin{align*} #1  \end{align*}}
\newcommand{\eqs}[1]{ \begin{equation*} #1  \end{equation*}}

\newcommand{\Lp}{\left(}
\newcommand{\Rp}{\right)}
\newcommand{\Lb}{\left[}
\newcommand{\Rb}{\right]}

\newcommand{\el}{\end{flushleft}}
\newcommand{\bl}{\begin{flushleft}}

\newcommand{\argmin}{\arg\!\min}
\newcommand{\argmax}{\arg\!\max}

\newcommand{\bE}{\mathbb{E}}

\newcommand{\bI}{\mathbb{I}}

\newcommand{\bP}{\mathbb{P}}

\newcommand{\bR}{\mathbb{R}}

\newcommand{\bV}{\mathbb{V}}

\newcommand{\cA}{\mathcal{A}}

\newcommand{\cC}{\mathcal{C}}
\newcommand{\cD}{\mathcal{D}}
\newcommand{\cE}{\mathcal{E}}
\newcommand{\cF}{\mathcal{F}}

\newcommand{\cL}{\mathcal{L}}

\newcommand{\cP}{\mathcal{P}}

\newcommand{\cX}{\mathcal{X}}

\newcommand{\cZ}{\mathcal{Z}}

\newcommand{\kR}{\mathfrak{R}}

\theoremstyle{plain}

\newtheorem{cor}{Corollary}

\newtheorem{rem}{Remark}

\newtheorem{assu}{Assumption}
\newtheorem{fact}{Fact}

%% file: latex/introduction.tex

Collecting human preference feedback is essential in many real-life applications, like online recommendations~\citep{AAAI13_kohli2013fast,TIS23_wu2023personalized,SIGIR23_zhang2023prompt,SIGKDD24_yang2024conversational}, content moderation~\citep{arXiv22_avadhanula2022bandits}, medical
treatment design~\citep{AM85_lai1985asymptotically,JMLR21_bengs2021preference},
prompt optimization \citep{arXiv24_lin2024prompt}, and aligning large language models \citep{bai2022training,menick2022teaching,arXiv23_mehta2023sample,chaudhari2024rlhf,arXiv24_das2024provably,arXiv24_ji2024reinforcement}, to ensure systems effectively align with user preferences and exhibit desired behaviors. 
However, this process is often costly due to the need for skilled evaluators, the complexity of tasks, and the time-intensive nature of producing high-quality, reliable human feedback. 
To address the challenge of balancing cost and effectiveness in aligning systems, this paper proposes principled and practical algorithms for efficiently collecting human feedback sequentially and adaptively to achieve the desired system behavior. 
Specifically, we aim to answer the following fundamental question:
\newline
{\bf \emph{How to achieve desired system behavior while using as minimum human feedback as possible?}}

Recent works \citep{arXiv23_mehta2023sample,arXiv24_das2024provably} have modeled the problem of active human feedback collection as an active version of the contextual dueling bandit problem (\adb{} for brevity)~\citep{NeurIPS21_saha2021optimal, ICML22_bengs2022stochastic, arXiv24_li2024feelgood}, where context-arm pair in the contextual dueling bandits corresponds to a task for which human preference feedback is collected and then proposed algorithms to select context-arm pairs for human feedback sequentially and adaptively by exploiting collected preference dataset, i.e., past context-arm pairs with their preference feedback. 
The preference feedback between two context-arm pairs is commonly assumed to follow the Bradley-Terry-Luce (BTL) model\footnote{For more than two context-arm pairs, preferences are typically modeled using the Plackett-Luce model \citep{ICML14_soufiani2014computing}.} \citep{AS04_hunter2004mm,ICML22_bengs2022stochastic, arXiv24_li2024feelgood, arXiv24_lin2024prompt, ICLR25_verma2025neural} in which the probability of preferring a context-arm pair over others is proportional to the exponential of its reward.
In many real-life applications, the number of context-arm pairs (e.g., user-movie pair in online movie recommendation) can be large or even infinite. Therefore, the reward for each context-arm pair is assumed to be an unknown function of its feature vector, such as a linear function \citep{arXiv23_mehta2023sample,arXiv24_das2024provably}.

To better align the system for optimal performance, we consider two key components: \emph{context selection} and \emph{arm selection}. 
The context selection aims to encourage diversity by exploring the context space, such as selecting prompts as diverse as possible in prompt optimization. 
Whereas arm selection focuses on identifying the arms that help learn the best arm for each context, such as selecting the most effective pair of responses to a given prompt that maximizes the system's learning \citep{arXiv24_lin2024prompt,ICLR25_verma2025neural}. 
Since the goal is to identify the best arm for each context, selecting suboptimal arms provides less useful information than choosing better arms. 
Existing methods for active contextual dueling bandits \citep{arXiv23_mehta2023sample,arXiv24_das2024provably} fail to incorporate an efficient arm selection strategy during the data collection process, thereby limiting the ability of these methods to achieve optimal performance.

An efficient arm selection strategy requires estimating the reward function to guide the arm selection process effectively.
Since the reward function may not always be linear in practice, this paper parameterizes the reward function via a \emph{non-linear function}, which needs to be estimated from the available preference dataset by using methods like Gaussian processes \citep{Book_williams2006gaussian, ICML10_srinival2010gaussian} or neural networks \citep{zhou2020neural, zhang2020neural}. 
However, Gaussian processes have limited expressive power and fail to optimize highly complex functions. In contrast, neural networks (NNs) have greater expressive power, making them well-suited for modeling complex functions \citep{ICLR23_dai2022federated, ArXiv23_lin2023use,arXiv24_lin2024prompt, ICLR25_verma2025neural}.

In this paper, we propose a neural active contextual dueling bandit algorithm, \algo{}, which uses an NN to estimate the unknown reward function using the available preference dataset. 
The context selection in \algo{} is adapted from \cite{arXiv24_das2024provably}, while arm selection strategies are based on, respectively, upper confidence bound (UCB) and Thompson sampling (TS), and adapted from \cite{ICLR25_verma2025neural}. 
Due to the differences in context selection strategy, arm selection strategies, and the use of a non-linear reward function, our theoretical analysis is completely different than related existing work \citep{arXiv23_mehta2023sample,arXiv24_das2024provably}. 
One of the key theoretical contributions of this paper is providing an upper bound on the maximum Mahalanobis norm of a vector from the fixed input space, measured with respect to the inverse of a positive definite Gram matrix that is constructed using finite, adapted samples from that space.
Building on this result, we prove that the worst sub-optimality gap (defined in \cref{eqn:suboptimality}) of the policy learned by \algo{} decreases at a sub-linear rate as the preference dataset size increases.

Specifically, our key contributions can be summarized as follows:
\vspace{-2mm}
\begin{itemize}
	\setlength\itemsep{-0.07em}
    \item We introduce the setting of active contextual dueling bandits with a non-linear reward function in \cref{sec:problem}. In \cref{sec:neural_adb}, we propose a neural active contextual dueling bandit algorithm, \algo{}, which uses an NN to estimate the unknown reward function from the available preference dataset and then uses this estimate into the arm selection strategies. 

    \item We prove an upper bound on the maximum Mahalanobis norm of a vector from the fixed input space, as measured with respect to the inverse of a positive definite Gram matrix (\cref{thm:normUB}), where the gram matrix is constructed using finite, adapted samples from that input space. We show that this upper bound decays at a sub-linear rate as the number of samples used in the Gram matrix increases. This theoretical result itself is of independent interest, as it gives valuable insights beyond the specific application of our work.
    
    \item We prove that the worst sub-optimality gap of the policy learned by \algo{} with both of our arm selection strategies (\cref{thm:subGapUCB} and \cref{thm:subGapTS}) decreases at a sub-linear rate with respect to the size of preference dataset, specifically at rate of  $\tilde{O}((\tilde{d}/T)^{\frac{1}{2}})$, where $\tilde{O}$ hides the logarithmic factors and constants, and $\tilde{d}$ is the effective dimension of context-arm feature vectors. The decay rate of the worst sub-optimality gap for \algo{} improves by a factor of $\tilde{O}(({\tilde{d}\log T})^{\frac{1}{2}})$ compared to exiting algorithms \citep{arXiv23_mehta2023sample, arXiv24_das2024provably}, thus bridging the gap between theory and practice.
        
    \item Finally, in \cref{sec:experiments}, our experimental results further validate the different performance aspects of \algo{}, highlighting its sample efficiency for preference data collection.
    
\end{itemize}

%% file: latex/problem.tex

We model active human preference feedback collection as an active contextual dueling bandit problem, where a labeler (human or simulator) provides preference feedback for a chosen pair of arms.

\para{Active contextual dueling bandit.}
We consider an active contextual dueling bandit problem, where the underlying latent reward function can be non-linear.
In each iteration of this problem, the learner's goal is to select a triplet containing a context and two arms for collecting preference feedback from a labeler/human such that the collected preference dataset leads to superior performance.
Let $\cC$ be the set of contexts and $\cA$ be the set of all possible arms. 
In each iteration, the learner selects a context $c_t \in \cX$ and then two arms (denoted as $a_{t,1}$ and $a_{t,2}$) from the set of arms $\cA$.
After selecting the triplet of context and two arms, the learner receives a stochastic preference feedback $y_t$, where $y_t = 1$ implies the arm $a_{t,1}$ is preferred over arm $a_{t,2}$ for the context $c_t$ and $y_t = 0$ otherwise.
We use $\phi(c_t,a)$ to denote the context-arm feature vector for context $c_t$ and an arm $a$, where $\phi:\cC \times \cA \rightarrow \R^d$ is a known feature map, such as one that concatenates the context and arm features.

\para{Preference model.}
Following the dueling bandits literature \citep{NeurIPS21_saha2021optimal, ICML22_bengs2022stochastic, arXiv24_li2024feelgood,ICLR25_verma2025neural}, we assume the preference feedback follows the Bradley-Terry-Luce (BTL) model\footnote{Our results are also applicable to any preference models, such as the Thurstone-Mosteller model and Exponential Noise, as long as stochastic transitivity holds \citep{ICML22_bengs2022stochastic}.} \citep{AS04_hunter2004mm, Book_luce2005individual}. 
Under the BTL preference model, the preference feedback has a Bernoulli distribution, where the probability that the first selected arm $a_{t,1}$ is preferred over the second selected arm $a_{t,2}$ for the given context $c_t$ is given by
\als{
    \Prob{a_{t,1} \succ a_{t,2}} &= \Prob{y_t = 1| c_t, a_{t,1}, a_{t,2}} = \mu\Lp f(\phi(c_t,a_{t,1})) - f(\phi(c_t,a_{t,2})) \Rp,
}
where $a_{t,1} \succ  a_{t,2}$ used for brevity and denotes that $a_{t,1}$ is preferred over $a_{t,2}$ for the given context $c_t$, 
$\mu(x) = 1/(1 + \mathrm{e}^{-x} )$ is the sigmoid function, $f: \R^d \rightarrow \R$ is an unknown non-linear bounded reward function, and $f(\phi(c,a))$ is the latent reward of the arm $a$ for the context $c$.
We require the following standard assumptions on the function $\mu$ (commonly referred to as a {\em link function} in the bandit literature \citep{ICML17_li2017provably,ICML22_bengs2022stochastic}):
\begin{assu}
\label{assup:link:function}
	\begin{itemize}
		\setlength{\itemsep}{3pt}
		\setlength{\parskip}{0pt}
		\item Let $\kappa_\mu \doteq \inf\limits_{c\in\cC, a,b \in \cA} \dot{\mu}(f(\phi(c,a)) - f(\phi(c,b))) > 0$ for all triplets of context $(c)$ and pair of arms $(a,b)$.
		\item The link function $\mu : \R \rightarrow [0,1]$ is continuously differentiable and Lipschitz with constant $L_\mu$. For logistic function, we have $L_\mu \le 1/4$.
	\end{itemize}	
\end{assu}
\textbf{Performance measure.}~
We denote the collected preference dataset up to $T$ iterations by $\cD_T = \{(c_s, a_{s,w}, a_{s,l}, y_s)\}_{s=1}^T$, where $a_{s,w} \succ a_{s,l}$ for the selected context $c_s$ in iteration $s$.
We aim to learn a policy, $\pi: \cC \rightarrow \cA$  from the collected preference dataset $\cD_T$ that achieves the worst sub-optimality gap across all contexts in $\cC$, which is defined as follows:
\eq{
    \label{eqn:suboptimality}
    \Delta_{\cD_T}^{\pi} = \max_{c \in \cC}\Lb\max_{a \in \cA} f(\phi(c,a)) - f\Lp\phi(c,\pi(c))\Rp \Rb,
}
where policy $\pi$ is a learned policy from the collected preference dataset $\cD_T$ up to the iteration $T$. 
The policy $\pi_{\cD_T}$ competes with the Condorcet winner~\citep{JMLR21_bengs2021preference, arXiv24_das2024provably} for a given context, i.e., an arm that is better than all other arms. The suboptimality gap is the worst possible difference in latent rewards over the set of contexts, and the same performance measure is used in prior work \citep{arXiv23_mehta2023sample,arXiv24_das2024provably}.

%% file: latex/neural_adb.tex

In this section, we introduce \algo{}, a simple yet principled and practical algorithm designed to efficiently select context-arm pairs for collecting preference feedback. 
\algo{} consists of two main components: context selection and arm selection. 
Since the arm selection strategy depends on the estimated reward function, we first explain how an NN can be used to estimate the unknown reward function. 
We will then give details of the context and arm selection strategies, followed by our theoretical results that validate the effectiveness of \algo{}.

\subsection{Reward function estimation using neural network}
For estimating the latent reward function, we use a fully connected neural network (NN) with depth $D \ge 2$, a hidden layer width $w$, and ReLU activations as done in \cite{zhou2020neural}, \cite{zhang2020neural}, and \cite{ICLR25_verma2025neural}.
Let $h(x;\theta)$ be the output of a full-connected NN with parameters $\theta$ for context-arm feature vector $x = \phi(c,a)$ of context $c$ and arm $a$, which we define as follows:
\eqs{
    h(x;\theta) = \bsym{W}_D \text{ReLU}\Lp \bsym{W}_{D-1} \text{ReLU}\Lp \cdots \text{ReLU}\Lp\bsym{W}_1 x\Rp \Rp \Rp,
}
where $\text{ReLU}(v) = \max\{v, 0 \}$, $\bsym{W}_1 \in \R^{w \times d}$, $\bsym{W}_l \in \R^{w \times w}$ for $2 \le l < D$, $\bsym{W}_D \in \R^{w \times 1}$. 
The parameters of the NN are represented by $\theta = \Lp \text{vec}\Lp \bsym{W}_1 \Rp;\cdots \text{vec}\Lp \bsym{W}_D \Rp \Rp$, where $\text{vec}\Lp A \Rp$ transforms an $m \times n$ matrix $A$ into a vector of dimension $mn$.
We use $p$ to represent the total number of NN parameters, which is given by $p = dw + w^2(D-1) + w$, and $g(x;\theta)$ to denote the gradient of NN $h(x;\theta)$ with respect to $\theta$.
At the end of each iteration $t$, the preference dataset $\cD_t = \left\{(c_s, a_{s,w}, a_{s,l}, y_s)\right\}_{s=1}^t$ is used to estimate the reward function $f$ by training an NN $h$ (parameterized by $\theta_{t+1}$) using gradient descent to minimize the following binary cross entropy loss function:
\al{
    \label{eqn:loss_function}
    \min_{\theta}\cL_t(\theta) &= 
    - \frac{1}{w} \sum^{t}_{s=1} \Big[ \log \mu \big(h(\phi(c_s,a_{s,w});\theta)\ -
    h(\phi(c_s,a_{s,l});\theta) \big)  \Big]    +  \frac{1}{2}\lambda \norm{\theta  - \theta_0}^2_{2},
}
where $\theta_0$ denotes the initial parameter of the NN that is initialized according to the standard practice in neural bandits \citep{zhou2020neural, zhang2020neural} (see Algorithm 1 in \cite{zhang2020neural} for details). 
Minimizing the first term in the above loss function (that involves the summation over the $t$ terms) corresponds to finding the maximum log-likelihood estimate of the parameters $\theta$.

\subsection{\algo{}}
We next propose a simple yet principled and practical algorithm, \algo{}, that consists of two key components: Context selection and arm selection. 
\algo{} works as follows: At the beginning of the iteration $t$, we first select the context as follows: 
\eq{
    \label{eqn:context_selection}
    c_t = \argmax_{c \in \cC} \max_{(a,b) \in \cA \times \cA} \norm{\phi(c, a) - \phi(c, b)}_{V_t^{-1}},
}
where $V_{t-1} = \frac{\lambda}{\kappa_{\mu}} \bI_p + \sum_{s=1}^{t-1} z_s {z_s}^\top \frac{1}{w}$ in which $z_s = \phi(c_s,a_{s,w}) - \phi(c_s,a_{s,l}) = g(\phi(c_s,a_{s,w});\theta_0) - g(\phi(c_s,a_{s,l});\theta_0)$, and $g(\phi(c_s,a_{s,i});\theta_0) / \sqrt{w}$ is used as the Random features approximation for the context-arm feature vector $\phi(c_s,a_{s,i})$.  
This strategy is adapted from the context selection strategy\footnote{Note that selecting contexts uniformly at random suffer a constant sub-optimality gap \citep[Theorem 3.2]{arXiv24_das2024provably}.} from \cite{arXiv24_das2024provably}.
After selecting context $c_t$, \algo{} uses the trained NN (as an estimate of the unknown reward function) to decide which two arms must be selected. 
To do so, \algo{} uses UCB- and TS-based arm selection strategies, which efficiently balance the trade-off between exploration and exploitation~\citep{Book_lattimore2020bandit} due to the bandit nature of preference feedback, as preference feedback is only observed for the selected pair of arms.
\begin{algorithm}[!ht]
	\renewcommand{\thealgorithm}{\algo{}}
	\floatname{algorithm}{}
	\caption{Neural Active Dueling Bandit algorithm}
	\label{alg:NeuralADB}
	\begin{algorithmic}[1]
		\STATE \textbf{Input parameters:} $\delta \in (0,1)$, $\lambda > 0$, and $w>0$
        \STATE \textbf{Initialize:} NN parameters $\theta_1$ and $D_0 = \emptyset$
		\FOR{$t= 1, \ldots, T$}
            \STATE Select a context $c_t$ from $\cC$ using \cref{eqn:context_selection}
            \STATE Select first arm $a_{t,1}$ using \cref{eqn:first_arm}
            \STATE Select second arm $a_{t,2}$ using \cref{eqn:second_arm_ucb} (for UCB-based arm selection) or \cref{eqn:second_arm_ts} (for TS-based arm selection)\hspace{-1mm}
            \STATE Observe preference feedback $y_t = \one{a_{t,1} \succ a_{t,2}}$
            \STATE Update $D_t = D_{t-1} \cup \{(c_t,a_{t,1},a_{t,2},y_t)\}$
            \STATE Retrain NN parameters $\theta_{t+1}$ using $D_t$ by minimizing the loss function defined in \cref{eqn:loss_function}
		\ENDFOR
        \STATE Return policy $\pi(c) = \argmax\limits_{a \in \cA} h(\phi(c,a);\theta_T),\ \forall c \in \cC$
	\end{algorithmic}
\end{algorithm}

\para{UCB-based arm selection strategy.}
Algorithms based on Upper confidence bound (UCB) are commonly used to address the exploration-exploitation trade-off in many sequential decision-making problems \citep{ML02_auer2002finite, NIPS11_abbasi2011improved,zhou2020neural,ICML22_bengs2022stochastic}. Our UCB-based arm selection strategy works as follows: In the iteration $t$, it selects the first arm greedily (i.e., by maximizing the output of the trained NN with parameters $\theta_t$) for the selected context $c_t$, ensuring the best-performing arm is always selected as follows:
\eq{
    \label{eqn:first_arm}
    a_{t,1} = \argmax_{a \in \cA} h(\phi(c_t,a);\theta_t).
}
The second arm $a_{t,2}$ is selected optimistically by maximizing the UCB value as follows: 
\eq{
    \label{eqn:second_arm_ucb}
    a_{t,2} = \mathop{\mathrm{argmax}}_{b \in \cA\setminus\{a_{t,1}\}} \Lb h(\phi(c_t,b);\theta_t) + \cf(t,c_t, a_{t,1}, b)\Rb,
}
where $\cf(t,c_t, a_{t,1}, b) = \nu_T \sigma_{t-1}(c_t,a_{t,1},b)$, $\nu_T \doteq (\beta_T + B \sqrt{\lambda / \kappa_\mu} + 1) \sqrt{\kappa_\mu / \lambda}$ in which $\beta_T \doteq \frac{1}{\kappa_\mu} \sqrt{ \widetilde{d} + 2\log(1/\delta)}$, $\widetilde{d}$ is the \emph{effective dimension} (defined in \cref{eqn:effectiveDim}), and 
\eq{
    \label{eqn:sigma}
    \sigma_{t-1}^2(c,a,b) \doteq \frac{\lambda}{\kappa_\mu} \norm{\frac{1}{\sqrt{w}}(\phi(c,a) - \phi(c,b))}^2_{V_{t-1}^{-1}}.
}
A larger value of $\sigma_{t-1}^2(c_t,a_{t,1}, b)$ implies that arm $b$ is significantly different from $a_{t,1}$, given the contexts and arm pairs already selected.
As a result, the second term in \cref{eqn:second_arm_ucb} makes the second arm different from the first arm which ensures exploration.

\para{TS-based arm selection strategy.}
Thompson sampling (TS) selects an arm based on its probability of being the best \citep{BIOMETRIKA33_thompson1933likelihood}. 
Several works \citep{NIPS11_chapelle2011empirical,ICML13_agrawal2013thompson,ICML17_chowdhury2017kernelized,arXiv24_li2024feelgood} have shown that TS empirically outperforms UCB-based bandit algorithms. 
Therefore, we also propose a TS-based arm selection strategy in which the first arm is also selected using \cref{eqn:first_arm} and the second arm is selected differently. 
To select the second arm, it first samples a score $s_t(b) \sim \mathcal{N}\big(h(\phi(c_t,b);\theta_t) - h(\phi(c_t,a_{t,1});\theta_t),\nu_T^2 \sigma_{t-1}^2(c_t, a_{t,1},b) \big)$ for every arm $b \in \cA\setminus\{a_{t,1}\}$ and then selects the second arm that maximizes the samples scores as follows:
\eq{
    \label{eqn:second_arm_ts}
    a_{t,2} = \mathop{\mathrm{argmax}}\nolimits_{b \in \cA\setminus\{a_{t,1}\}} s_t(b).
}

After selecting context and arms in iteration $t$, stochastic preference feedback is observed, denoted by $y_s = \one{a_{t,1}\succ a_{t,2}}$, which is equal to $1$ if arm $a_{s,1}$ is preferred over arm $a_{s,2}$ for context $c_t$ and $0$ otherwise. 
With the new observation, the preference dataset is updated to $\cD_t = \cD_{t-1} \cup \{(c_t, a_{t,w}, a_{t,l}, y_t)\}$ and then the NN is retrained using the updated preference dataset $\cD_t$. 
Once the preference data collection process concludes (i.e., end of iteration $T$, which may not be fixed a priori), \algo{} returns the following policy:
\eq{
    \label{eqn:policy}
    \forall c \in \cC:\ \pi(c) = \argmax_{a \in \cA} h(\phi(c,a);\theta_T).
}

%% file: latex/theory_results.tex

Let the number of arms in $\cA$ be finite, and define $\bV = \sum_{s=1}^T \sum_{(a, b) \in \cA \times \cA} z_{a,b}(s) z_{a,b}(s)^\top  \frac{1}{w}$, where $z_{a,b}(s) = \phi(c_s, a) - \phi(c_s,b)$ and $C^{|\cA|}_2$ denotes all pairwise combinations of arms. Then, the \emph{effective dimension} of context-arm feature vectors is defined as follows:
\eq{
    \label{eqn:effectiveDim}
    \widetilde{d} = \log \det  \left(\frac{\kappa_\mu}{\lambda} \bV + \bI_p\right).
}
In the following, we present a novel theoretical result that gives an upper bound on the maximum Mahalanobis norm of a vector selected from the fixed input space, measured with respect to the inverse of a positive definite Gram matrix constructed from finite, adapted samples of the same space.\hspace{-1mm}
\begin{restatable}{thm}{normUB}
    \label{thm:normUB}
    Let $\{Z_s = z_sz_s^\top\}_{s=1}^T$ be a finite adapted sequence of self-adjoint matrices in $\R^d$. Define $\EE{z_sz_s^\top} = \Sigma_s \le \Sigma_{\max}$, $V_0 = \lambda \bI_d$, $V_T = \lambda \bI_d + \sum_{s=1}^T z_sz_s^\top$. Assume $\norm{z_s}_2 \le L$ for all $z \in \cZ \subset \R^d$, $\lambda_{\min}(A)$ denote the minimum eigenvalue of a matrix $A$, and $\forall s \le T:\ \norm{V_s - V_{s-1}}^2 \le C_s$, where $\norm{V}$ denotes the operator norm. 
    Then, with a probability at least $1-\delta$,
    $
        \max_{z \in \cZ}\norm{z}_{V_T^{-1}} \le L/G_T,
    $
    where $G_T = \sqrt{T\lambda_{\min}(\Sigma_{\max}) -\sqrt{8\sum_{s=1}^T C_s \log \Lp \nicefrac{d}{\delta} \Rp}}$. 
\end{restatable}
\para{Proof sketch.} 
To derive the upper bound, we use various results related to the positive definite matrix (detailed in \cref{fact:pdMatrixProp} of the supplementary material). 
First, if $V_T$ is a positive definite matrix $V_T$, then for any $z \in \cZ$, $\norm{z}_{V_T^{-1}} \le \norm{z}_2\sqrt{\lambda_{\max}(V_n^{-1})} = \norm{z}_2/\sqrt{\lambda_{\min}(V_n)}$. 
Thus, $\max\limits_{z \in \cZ}\norm{z}_{V_T^{-1}} \le \norm{z}_2/\sqrt{\lambda_{\min}(V_n)} \le L/\sqrt{\lambda_{\min}(V_n)}$. 
Since $\{Z_s\}_{s=1}^T$ is a finite adapted sequence of self-adjoint matrices (i.e., $Z_s$ is $\cF_s$-measurable for all $s$, where $\cF_s$ represents all information available up to iteration $s$), we apply the Matrix Azuma inequality \citep{FCM12_tropp2012user} to get a high probability lower bound on $\lambda_{\min}(V_n)$, specifically we have shown that $\lambda_{\min}(V_T) \ge T\lambda_{\min}(\Sigma_{\max}) -\sqrt{8\sum_{s=1}^T C_s \log \Lp \nicefrac{d}{\delta} \Rp}$ holds with probability at least $1-\delta$. Using this bound, we get the desired upper bound $L/G_T$. The full proof of \cref{thm:normUB}, along with all other missing proofs, are provided in \cref{sec:appendix}. 

This result shows that the upper bound can be expressed in terms of the number of adapted samples used to construct the matrix $V_T$, and it decays at a sub-linear rate as the number of samples $(T)$ increases. Notably, this result is of independent interest, as it provides valuable insights beyond the specific application of our work.
Next, we give an upper bound on the worst sub-optimality gap in terms of the upper bound on the estimation error of the reward difference between any triplet consisting of a context and two arms.

\begin{restatable}{lem}{subGapAbsUB}
    \label{lem:subGapAbsUB}
    Let $D_T = \{x_s, a_{s,1}, a_{s,2}, y_s\}_{s=1}^T$ be the preference dataset collected up to the iteration $T$ and $\hat{f}_T$ represent the estimate of latent reward function $f$ learned from $D_T$. With probability at least $1-\delta$, $\forall c\in\cC,\ a, b \in \cA:\ { \left|\Lb f(\phi(c,a)) - f\Lp\phi(c,b)\Rp\Rb - \Lb\hat{f}_T(\phi(c,a)) - \hat{f}_T(\phi(c,b)) \Rb \right|}$ $\le \beta_T(c,a,b)$. If  $a^\star = \argmax_{a \in \cA} f(\phi(c,a))$ and $\pi(c)$ is the arm selected by policy for context $c$, then, with a probability at least $1-\delta$, the worst sub-optimality gap for a policy that greedily selects an arm for a given context is upper bounded by:
    $
        \Delta_T^\pi \le \max\limits_{c \in \cC} \beta_T(c, a^\star, \pi(c)). 
    $
\end{restatable}
The proof follows by starting with the worst sub-optimality gap definition in \cref{eqn:suboptimality} and then applying a series of algebraic manipulations to derive the stated result. Our next results give an upper bound on $\beta_T(c,a,b)$ when \algo{} uses different arm selection strategies.
\begin{restatable}{lem}{confBounds}
    \label{lem:confBounds}
    Let $\nu_T = (\beta_T + B \sqrt{\lambda / \kappa_\mu} + 1) \sqrt{\kappa_\mu / \lambda}$, where $\beta_T = \Lp\nicefrac{1}{\kappa_\mu}\Rp \sqrt{ \widetilde{d} + 2\log(1/\delta)}$ and
    $\delta\in(0,1)$.
    If $w \geq \text{poly}(T, L, K, 1/\kappa_\mu, L_\mu, 1/\lambda_0, 1/\lambda, \log(1/\delta))$, then, with a probability of at least $1-\delta$, for \algo{} with
    \vspace{-3mm}
    \begin{enumerate}
    	\setlength\itemsep{-0.07em}
        \item UCB-based arm selection strategy, for all $c\in\cC:$
        \eqs{
            \beta_T(c,a,b) = \nu_T \sigma_{T}(c, a^\star, \pi(c)) + 2\varepsilon'_{w,T},
        }

        \item TS-based arm selection strategy, for all $c\in\cC:$
        \eqs{
            \beta_T(c,a,b) = \nu_T \log\Lp KT^2\Rp \sigma_{T}(c, a^\star, \pi(c)) + 2\varepsilon'_{w,T},
        }
    \end{enumerate} 
    where $K$ denotes the maximum number of arms available in each iteration, and $\varepsilon'_{w,T} = C_2 w^{-1/6}\sqrt{\log w} L^3 \left(\nicefrac{T}{\lambda}\right)^{4/3}$ for some absolute constant $C_2>0$, is the approximation error that decreases as the width of the NN $(w)$ increases. 
\end{restatable}
Equipped with \cref{thm:normUB}, \cref{lem:subGapAbsUB}, and \cref{lem:confBounds}, we will now provide an upper bound on the worse sub-optimality gap for a policy learned by \algo{} while using UCB- and TS-based arm selection strategy for a given context.

\begin{restatable}[UCB]{thm}{subGapUCB}
    \label{thm:subGapUCB}    
    Let the conditions in \cref{thm:normUB} and \cref{lem:confBounds} hold. 
    Then, with a probability with at least $1-\delta$, the worst sub-optimality gap of \algo{} when using UCB-based arm selection strategy is upper bounded by \vspace{-3mm}
    \als{
        \Delta_T^\pi &\le \Lp \frac{\nu_TL}{G_T} \Rp\sqrt{\frac{\lambda}{\kappa_\mu w}}  + 2\varepsilon'_{w,T} = \tilde{O} \Lp \sqrt{\frac{\tilde{d}}{T}}\Rp.
    }
\end{restatable}
\vspace{-2mm}
\begin{restatable}[TS]{thm}{subGapTS}
    \label{thm:subGapTS}
    Let the conditions in \cref{thm:normUB} and \cref{lem:confBounds} hold. 
    Then, with a probability with at least $1-\delta$, the worst sub-optimality gap of \algo{} when using TS-based arm selection strategy is upper bounded by \vspace{-3mm}
    \als{
        \Delta_T^\pi &\le \Lp\hspace{-1mm} \frac{\nu_TL\log \Lp KT^2 \Rp}{G_T} \hspace{-1mm}\Rp \hspace{-1mm}\sqrt{\hspace{-0.5mm}\frac{\lambda}{\kappa_\mu w}}+ 2\varepsilon'_{w,T} = \tilde{O} \Lp \hspace{-1mm}\sqrt{\hspace{-0.5mm}\frac{\tilde{d}}{T}} \Rp.
    }
\end{restatable}
\vspace{-2mm}

The proof follows by applying \cref{lem:confBounds}, setting $z = \phi(c,a^\star) - \phi(c,\pi(c))$ in \cref{eqn:sigma}, and then using \cref{thm:normUB}.
Note that $\varepsilon'_{w,T} = O(1/T)$ and $\widetilde{d} = \widetilde{o}(\sqrt{T})$ as long as the NN width $w$ is large enough \citep{zhou2020neural,zhang2020neural,ICLR25_verma2025neural}.
Above \cref{thm:subGapUCB} and \cref{thm:subGapTS} show that the worst sub-optimality gap of the policy learned by \algo{} with UCB- and TS-based arm selection strategies decreases at a sub-linear rate with respect to the size of preference dataset, specifically at rate of  $\tilde{O}((\tilde{d}/T)^{\frac{1}{2}})$, where $\tilde{O}$ hides the logarithmic factors and constants.
Further, the decay rate of the worst sub-optimality gap for \algo{} improves by a factor of $\tilde{O}(({\tilde{d}\log T})^{\frac{1}{2}})$ compared to exiting algorithms \citep{arXiv23_mehta2023sample, arXiv24_das2024provably}, thereby bridging the gap between theory and practice.

%% file: latex/adb_regret.tex

We start by defining the \emph{cumulative regret} (or `regret' for brevity) of a policy. After receiving preference feedback for $T$ pairs of arms, the regret of a sequential arm selection policy is given by:
$
	\Regret_T = \sum_{t=1}^T \Lb f(\phi(c_t, a_t^\star)) - \Lp{f(\phi(c_t,a_{t,1})) + f(\phi(a_{t,2}))}\Rp/{2}\Rb,
$
where $a_t^\star = \argmax_{a \in \cA_t}  f(\phi(c_t, a))$ denotes the arm that maximizes the reward function for a given context $c_t$.

In many real-world applications, such as medical treatment design~\citep{AM85_lai1985asymptotically,JMLR21_bengs2021preference} and content moderation~\citep{arXiv22_avadhanula2022bandits}, both actively selecting arms and minimizing regret is required. For instance, in personalized medical treatment, active learning is used to identify the most informative treatments to test, while cumulative regret minimization ensures the system continually adapts to deliver better patient outcomes.
Such scenarios also arise in other fields, such as dynamic pricing and personalized education, enabling systems to make smarter decisions, reduce suboptimal choices, and optimize overall performance as they gather more valuable data.

Since the arm selection strategies in \algo{} are directly adapted from UCB- and TS-based algorithms for contextual dueling bandits of \cite{ICLR25_verma2025neural}, the regret upper bounds for these algorithms also apply to \algo{}. For completeness, we state the regret upper bounds of \algo{} as follows.

\begin{cor}[Regret Upper Bound]{\citep[Theorem 2 and Theorem 3]{ICLR25_verma2025neural}}
    \label{cor:regretUCB}
    Let $\lambda > \kappa_\mu$ and $w \geq \text{poly}(T, L, K, 1/\kappa_\mu, L_\mu, 1/\lambda_0, 1/\lambda, \log(1/\delta))$. Then, with a probability of at least $1-\delta$, 
    the regret of \algo{} when using UCB- or TS-based arm selection strategy is upper bounded by 
    \eqs{
        \Regret_T = \widetilde{O}\left( \left( \frac{\sqrt{\widetilde{d}}}{\kappa_\mu} + \sqrt{\frac{\lambda}{\kappa_\mu}}\right) \sqrt{T  \widetilde{d} } \right).
    }
\end{cor}
Ignoring logarithmic factors and constants, the asymptotic growth rates of \algo{} with UCB- and TS-based arm selection strategy are identical and sub-linear.

%% file: latex/experiment.tex

To validate our theoretical results, we empirically evaluate the performance of our algorithms on different problem instances of synthetic datasets. 
Specifically, we use two commonly used synthetic functions adopted from existing works on neural bandits \citep{zhou2020neural, zhang2020neural,ICLR23_dai2022federated,ICLR25_verma2025neural}: $f(x)=10(x^\top\theta)^2$ (Square) and $f(x) = 2\sin(x^\top\theta)$ (Sine). 
All experiments are repeated 10 times, and we report the average worst suboptimality gap with 95\% confidence intervals (depicted as vertical lines on each curve).

\begin{figure}[!ht]
    \vspace{-5mm}
	\centering
    \subfloat[Suboptimality Gap]{\label{fig:subg_square}
		\includegraphics[width=0.29\linewidth]{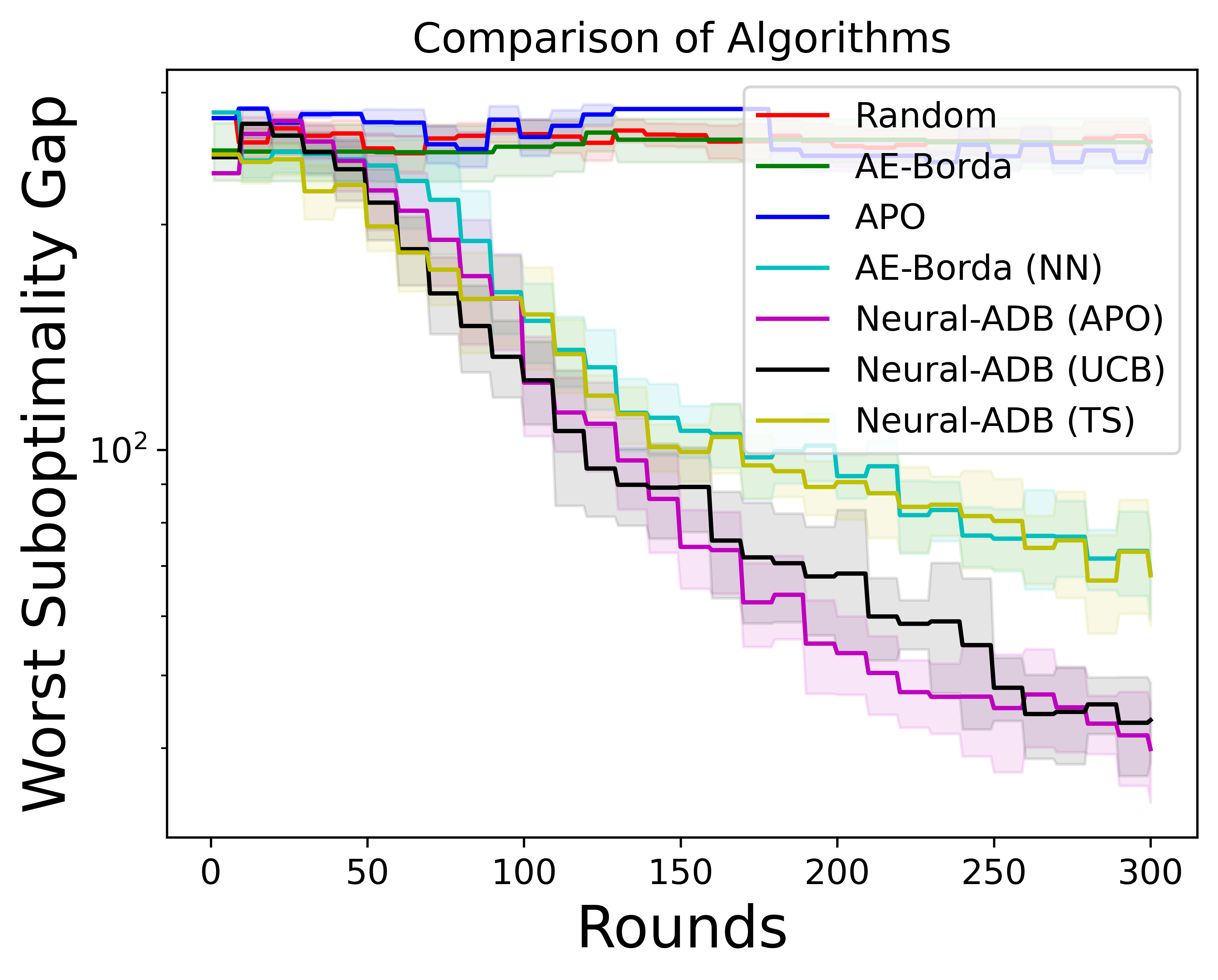}}
    \qquad 
    \subfloat[MAE]{\label{fig:subg_avg_square}
		\includegraphics[width=0.29\linewidth]{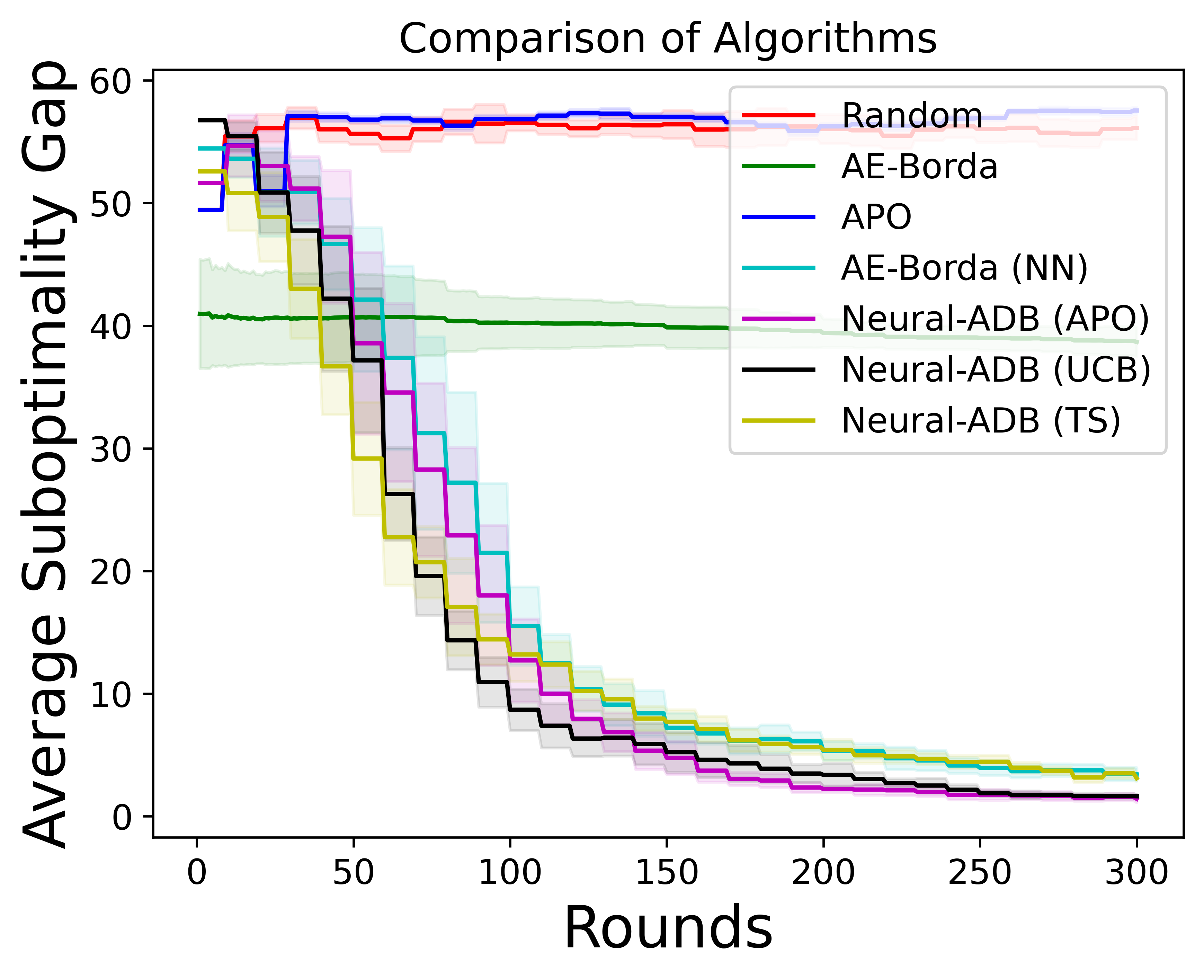}}
	\qquad 
    \subfloat[Average Regret]{\label{fig:rmse_avg_square}
		\includegraphics[width=0.29\linewidth]{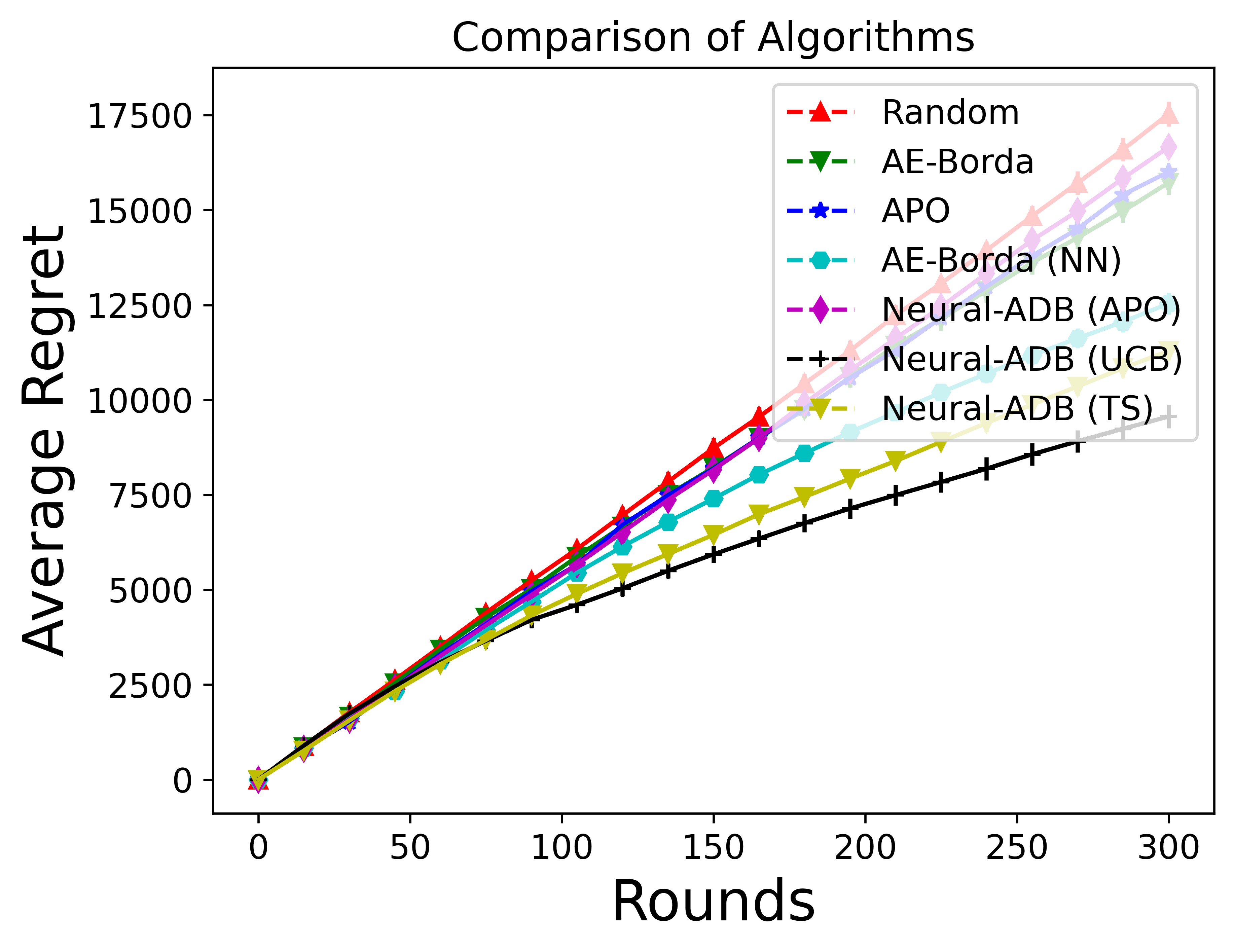}} 
    \\ \vspace{-1mm}
    \subfloat[Suboptimality Gap]{\label{fig:subg_sine}
		\includegraphics[width=0.29\linewidth]{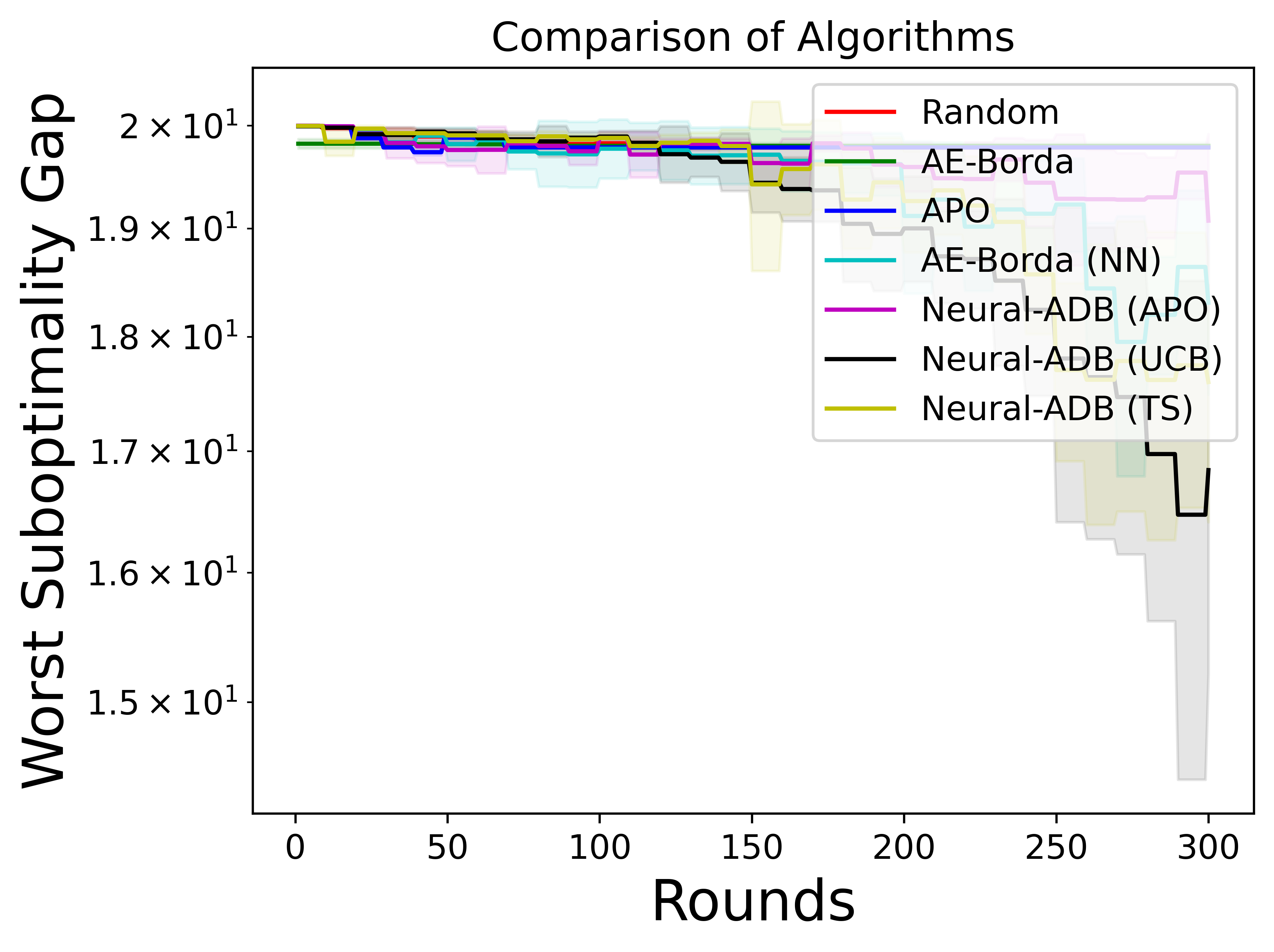}}
    \qquad 
    \subfloat[MAE]{\label{fig:subg_avg_sine}
		\includegraphics[width=0.29\linewidth]{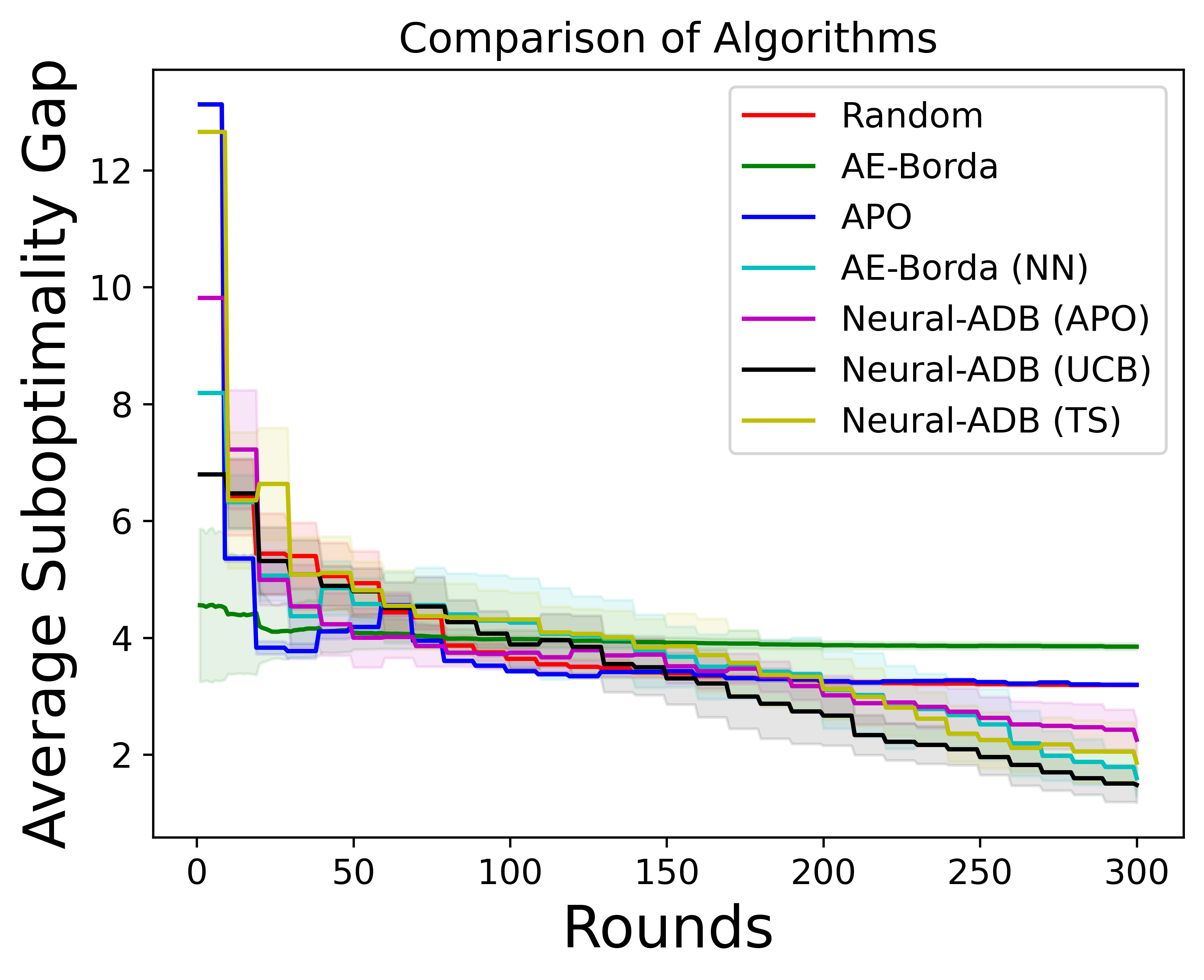}}
	\qquad 
    \subfloat[Average Regret]{\label{fig:rmse_avg_sine}
		\includegraphics[width=0.29\linewidth]{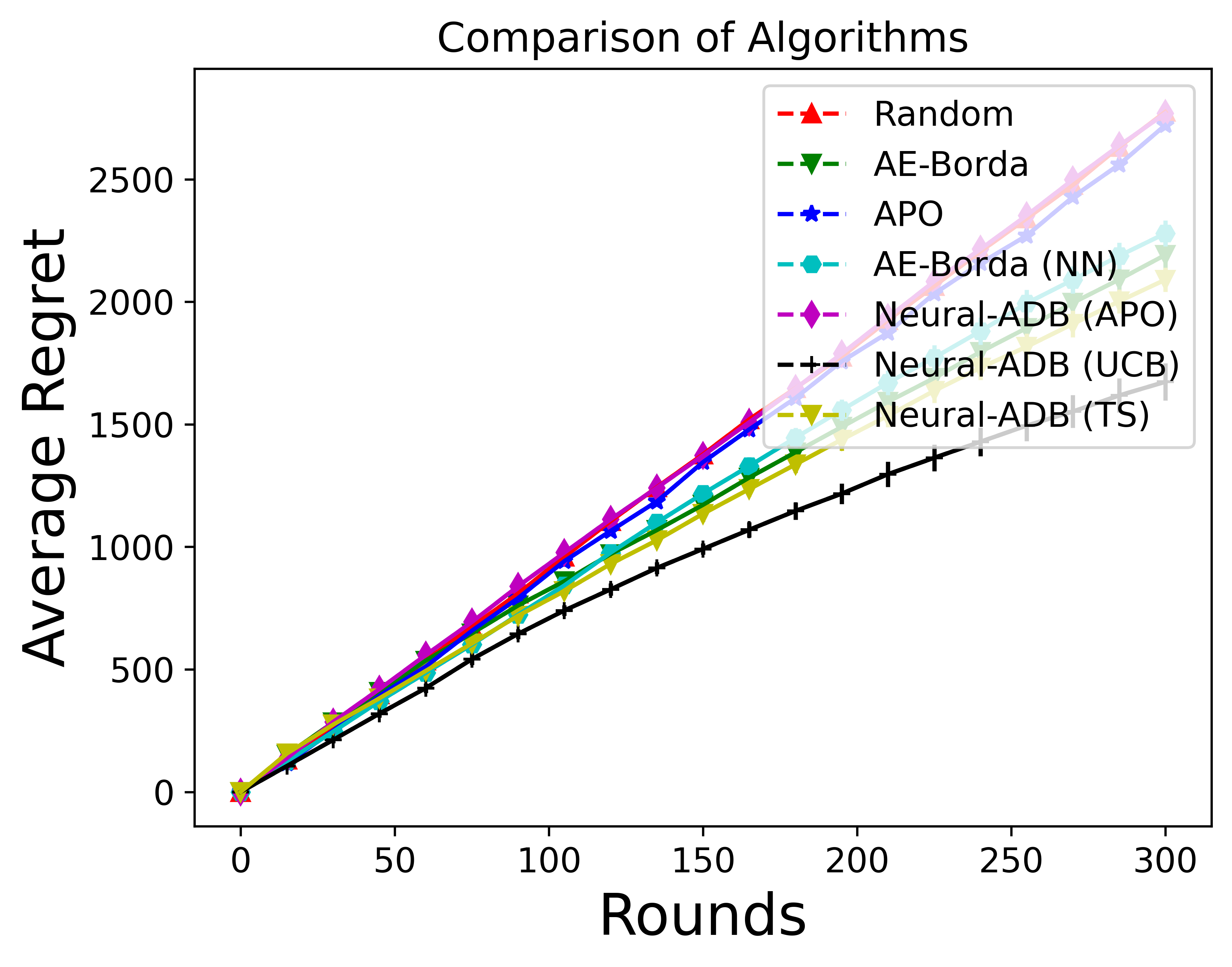}}  
	\caption{
         Performance comparison of \algo{} against different active dueling bandit algorithms on synthetic functions: Square function (top row) and Sine function (bottom row).
	}
	\label{fig:main-comparison}
    \vspace{-2mm}
\end{figure}

\begin{figure}[!ht] 
	\vspace{-2mm}
	\centering
    \subfloat[Suboptimality Gap]{\label{fig:subg_abl_dimsquare}
		\includegraphics[width=0.27\linewidth]{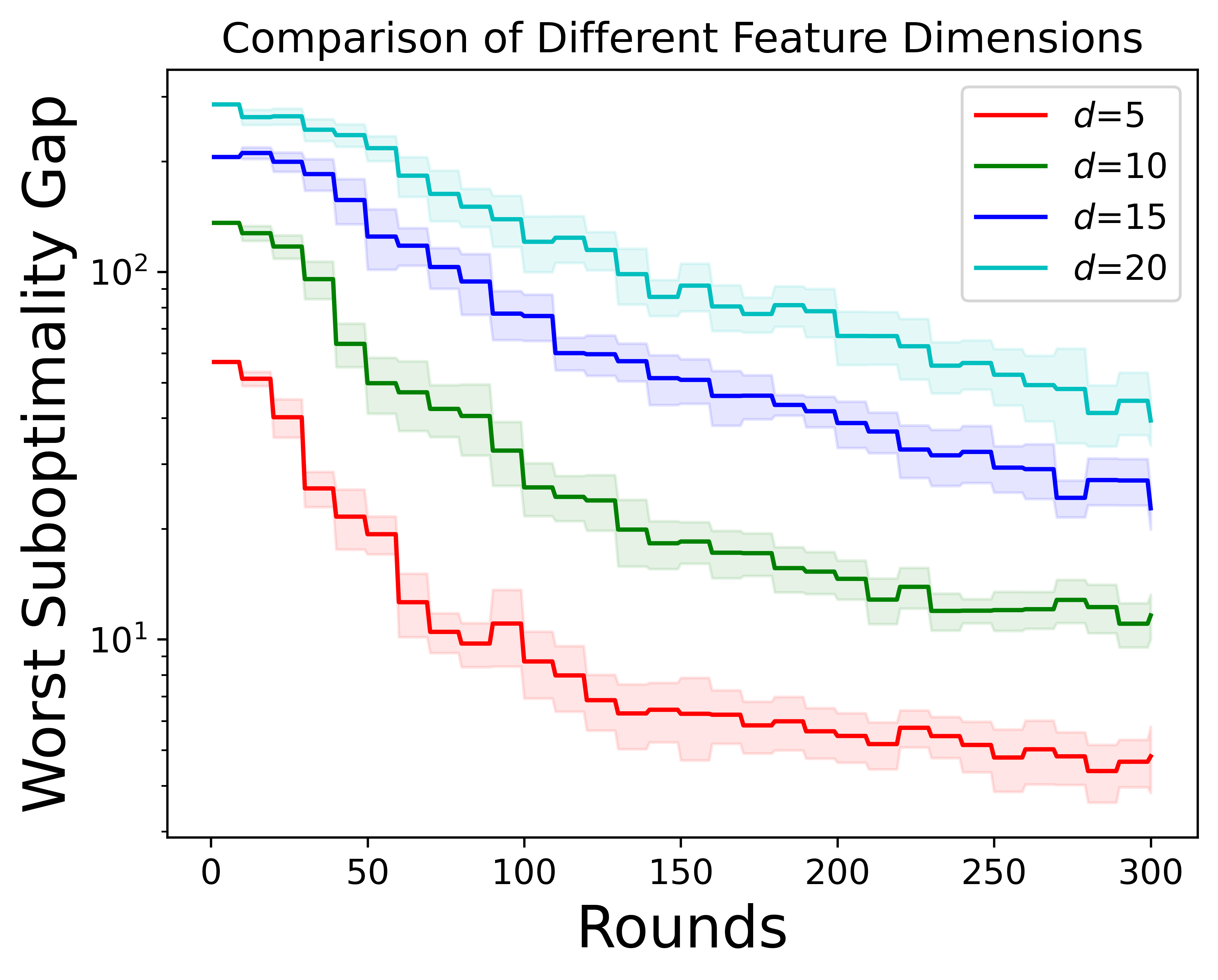}}
    \qquad 
    \subfloat[MAE]{\label{fig:subg_avg_abl_dimsquare}
		\includegraphics[width=0.27\linewidth]{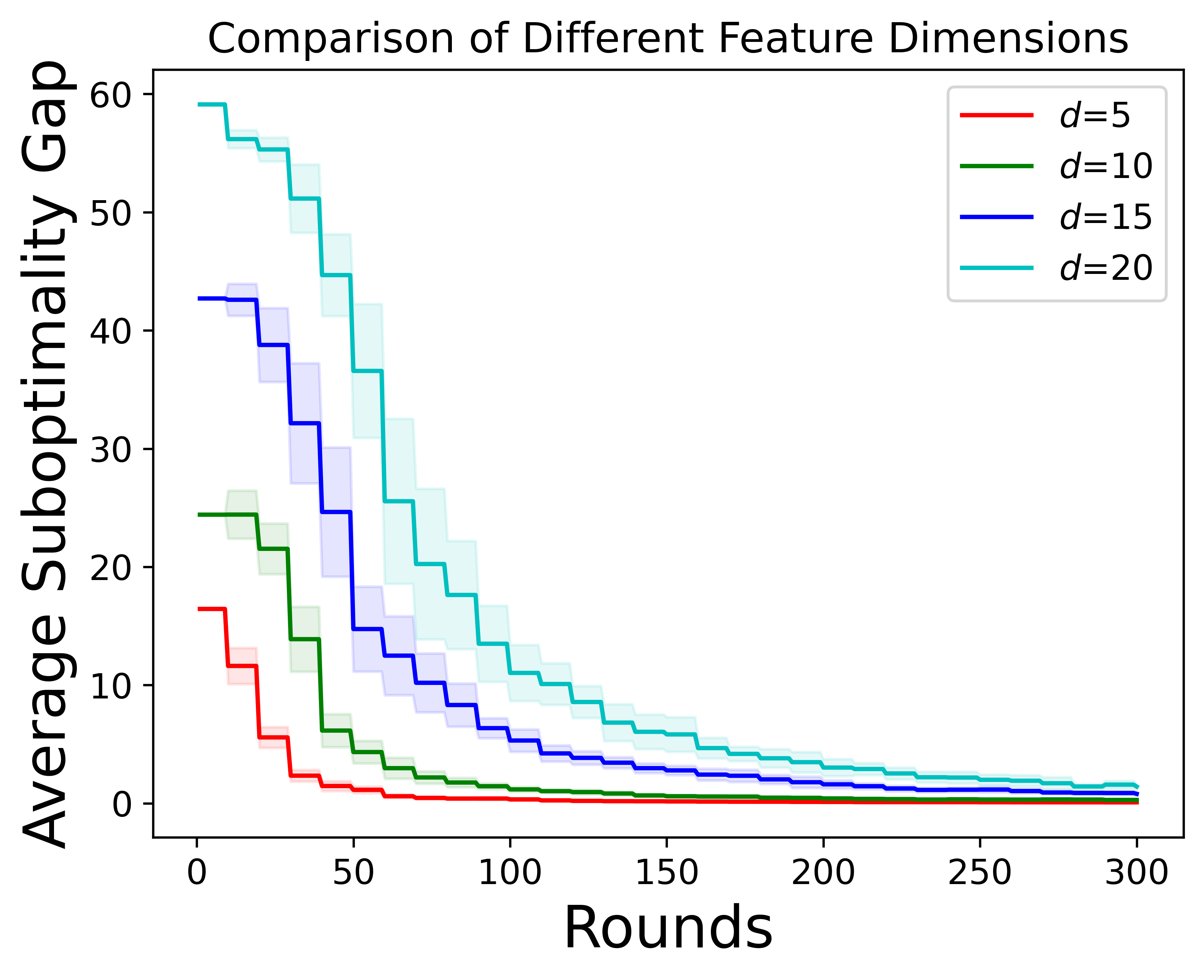}}
	\qquad 
    \subfloat[Average Regret]{\label{fig:rmse_avg_abl_dimsquare}
		\includegraphics[width=0.27\linewidth]{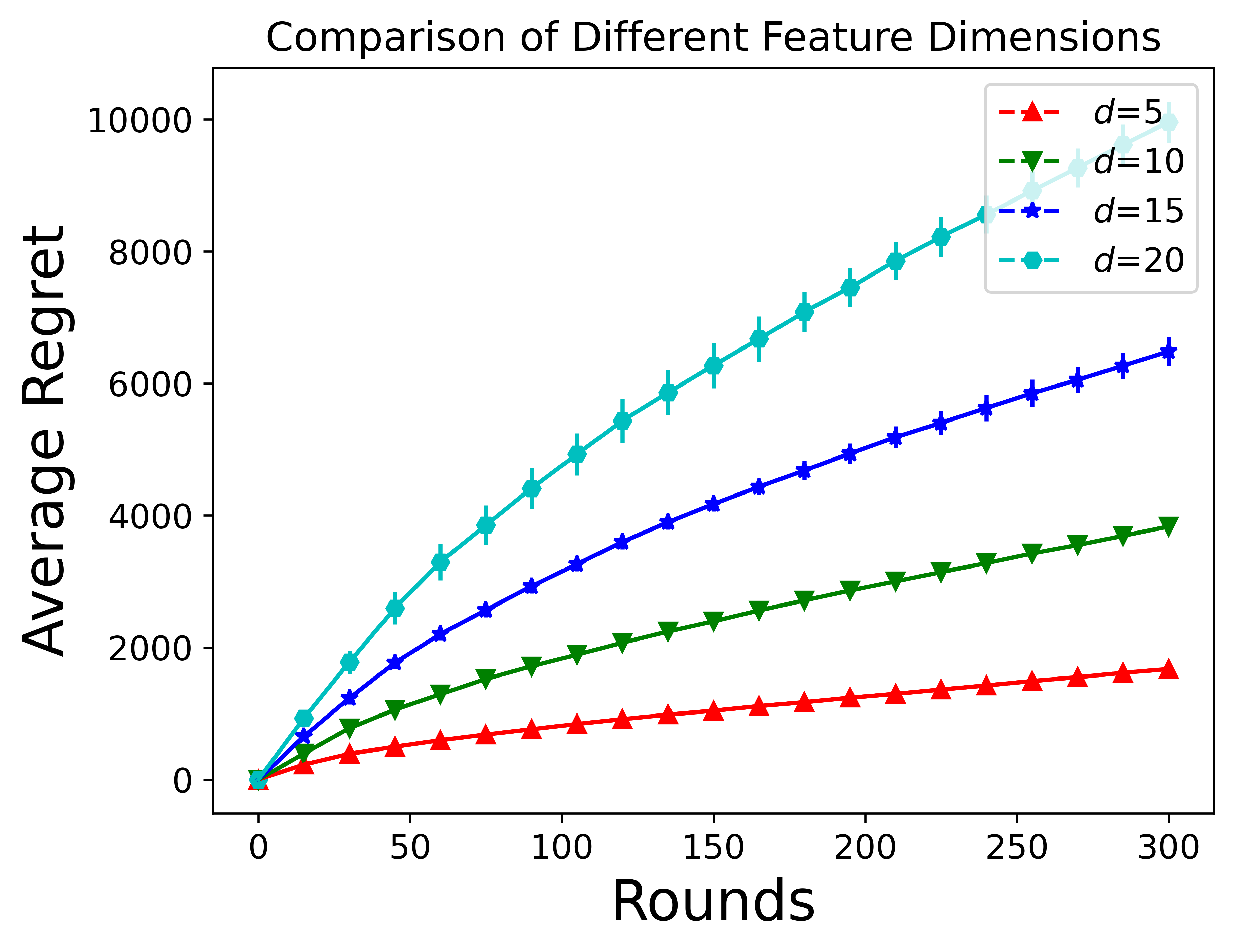}} \\
    \vspace{-1mm}
    \subfloat[Suboptimality Gap]{\label{fig:subg_abl_square}
		\includegraphics[width=0.27\linewidth]{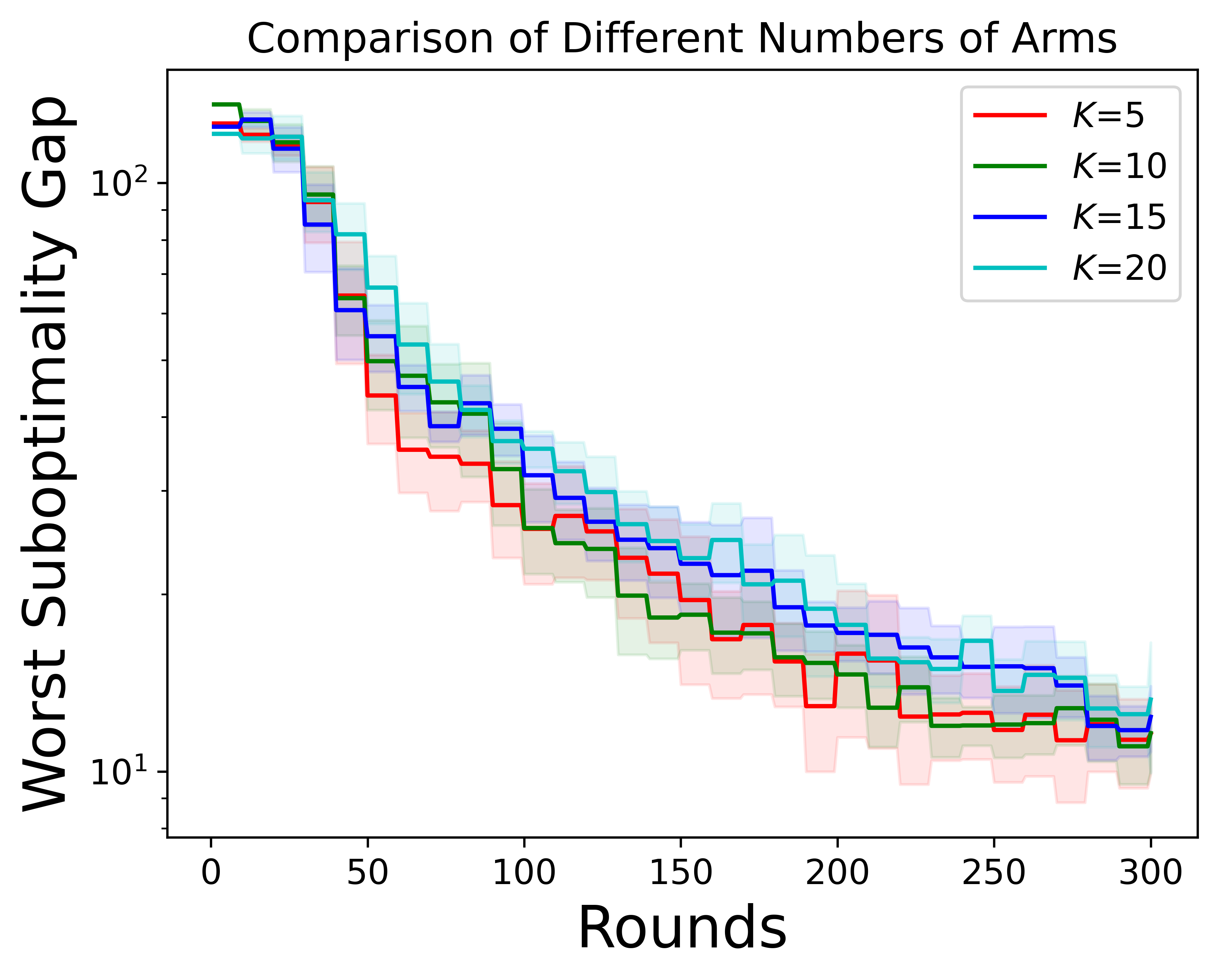}}
    \qquad 
    \subfloat[MAE]{\label{fig:subg_avg_abl_square}
		\includegraphics[width=0.27\linewidth]{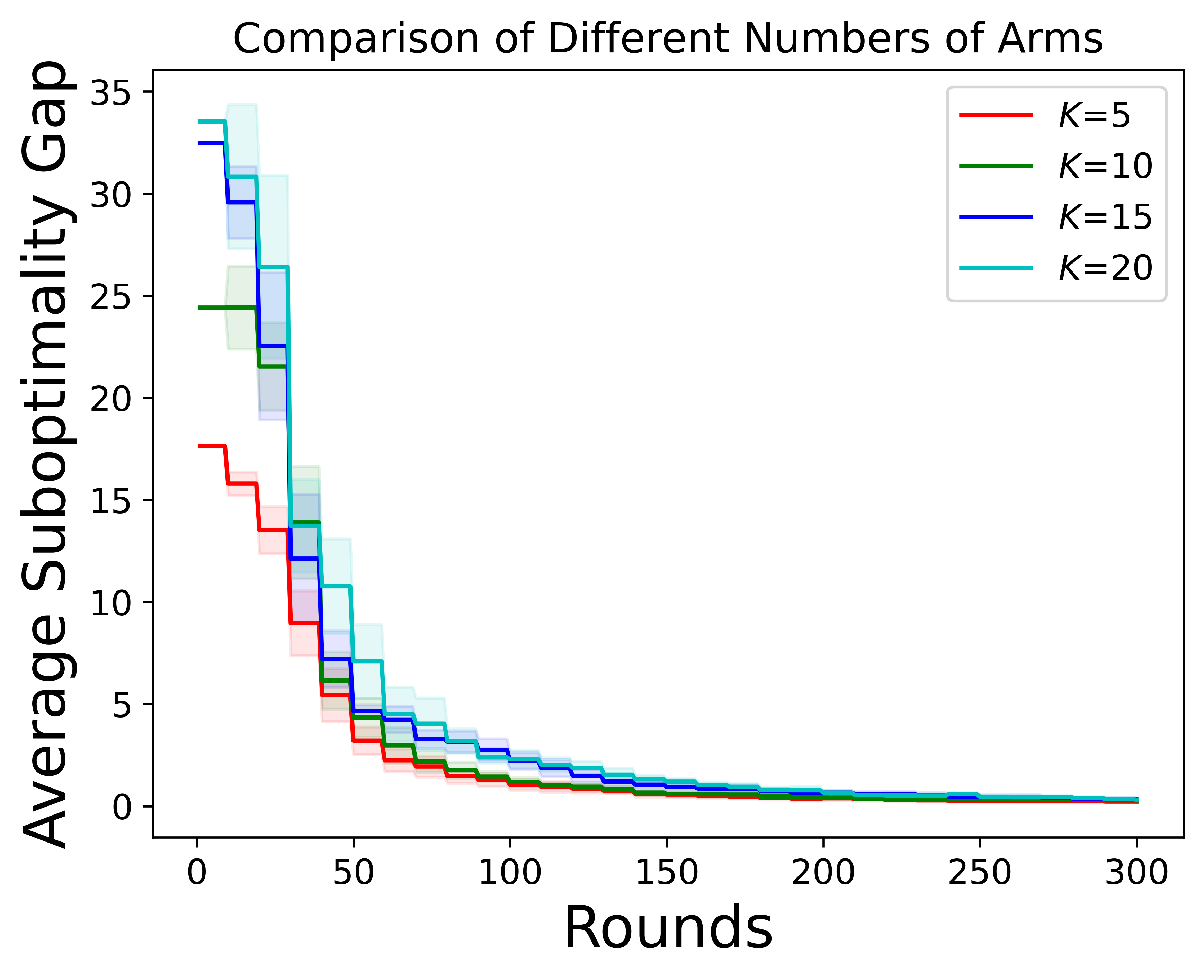}}
	\qquad 
    \subfloat[Average Regret]{\label{fig:rmse_avg_abl_square}
		\includegraphics[width=0.27\linewidth]{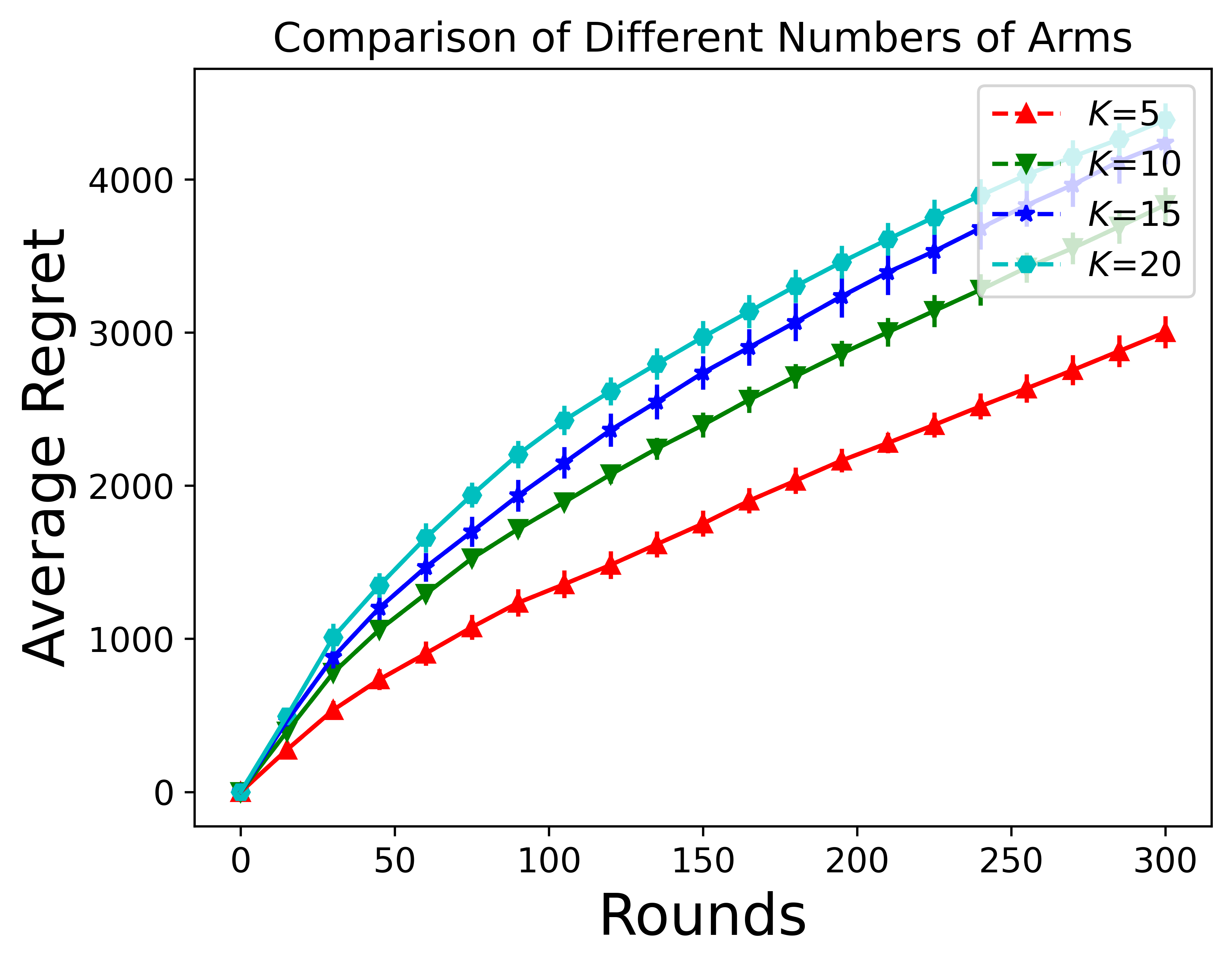}} 
	\caption{
         Performance of \algo{} (UCB) on the Square function, evaluated across varying input dimensions (top row) and numbers of arms (bottom row).
	}
	\label{fig:neuraldb_ucb_ablations}
    \vspace{-4mm}
\end{figure}

\begin{figure}[!ht]
    \vspace{-5mm}
	\centering
    \subfloat[Sub-Optimality Gap]{\label{fig:subg_dim}
		\includegraphics[width=0.27\linewidth]{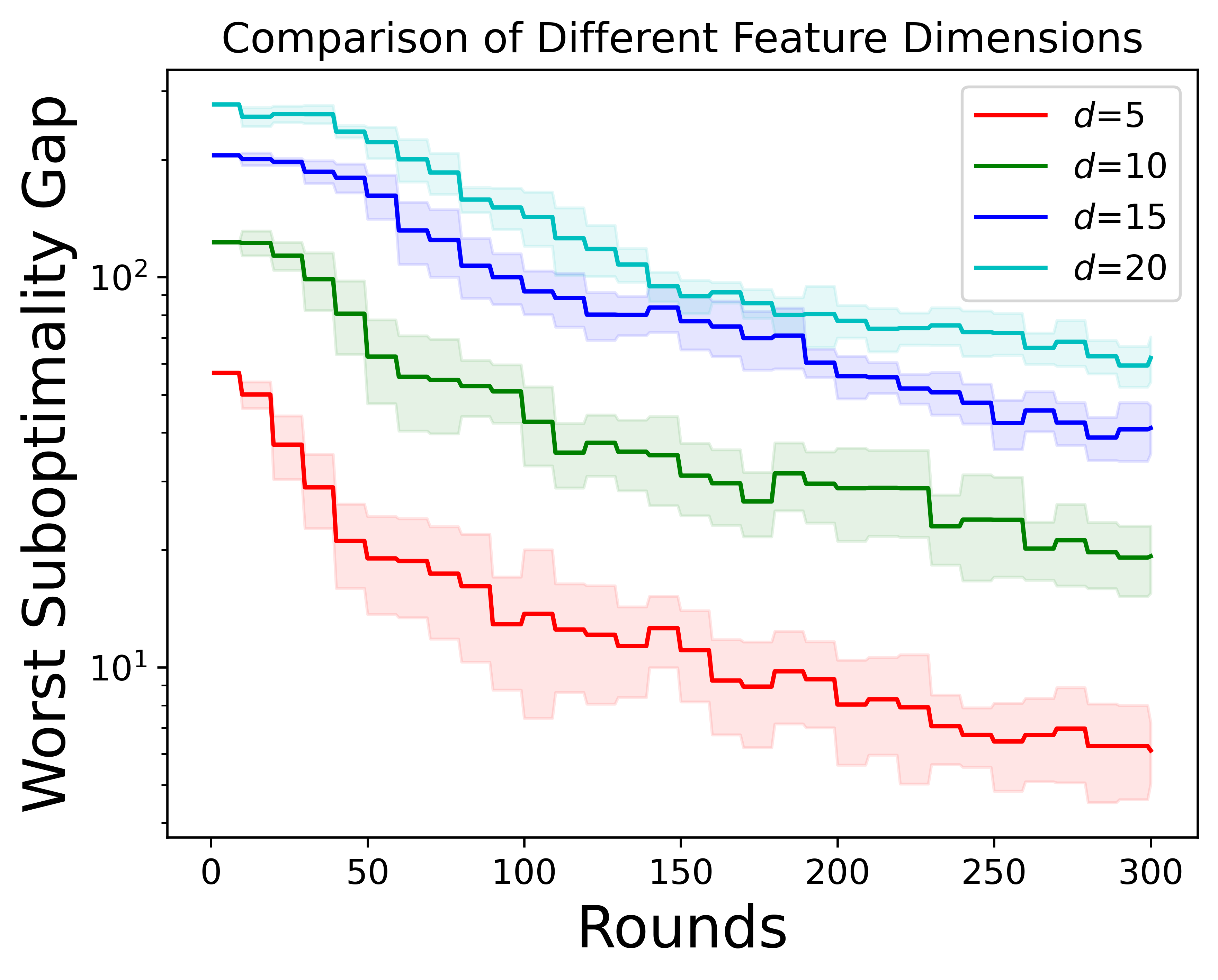}}
    \qquad
    \subfloat[MAE]{\label{fig:subg_avg_dim}
		\includegraphics[width=0.27\linewidth]{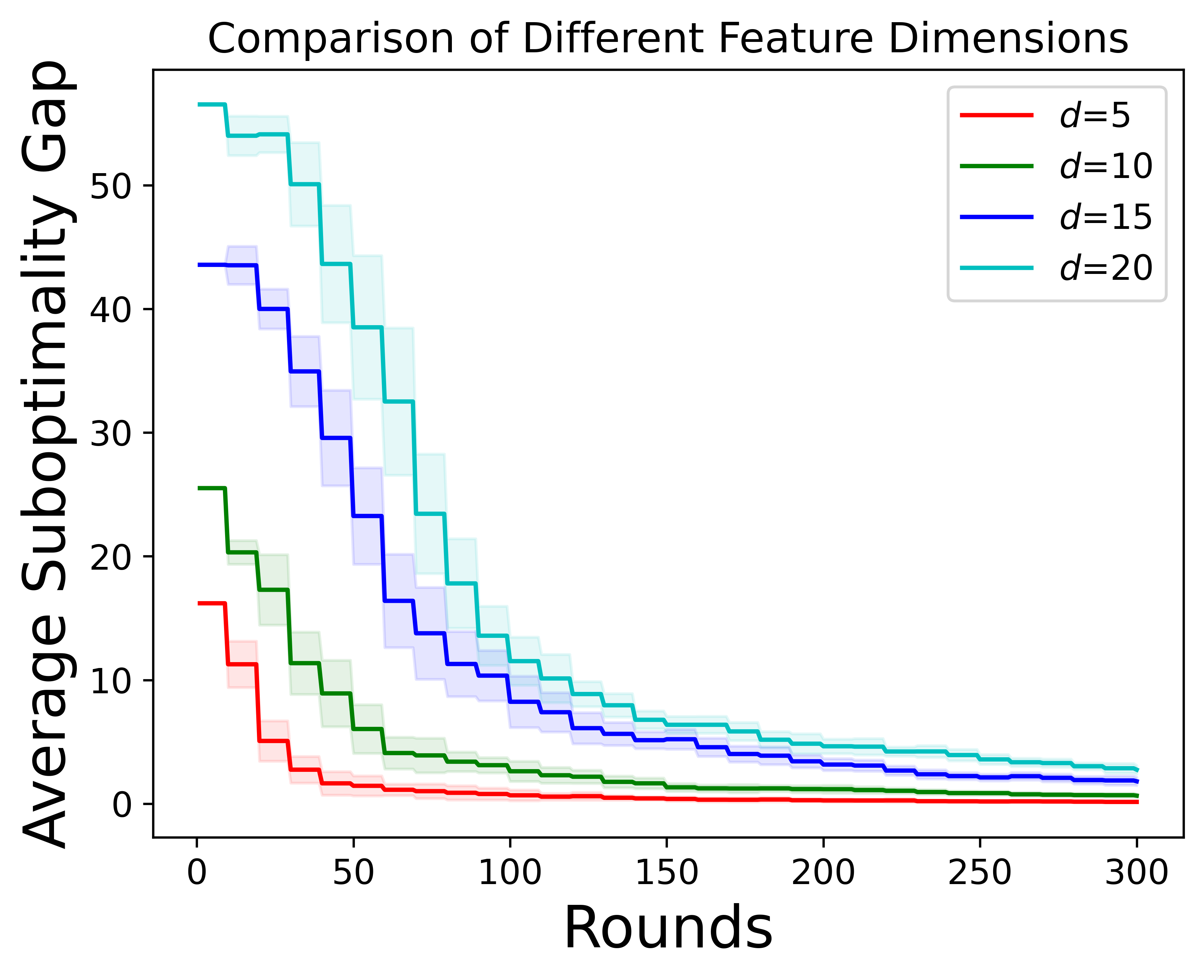}}
	\qquad
    \subfloat[Average Regret]{\label{fig:rmse_avg_dim}
		\includegraphics[width=0.27\linewidth]{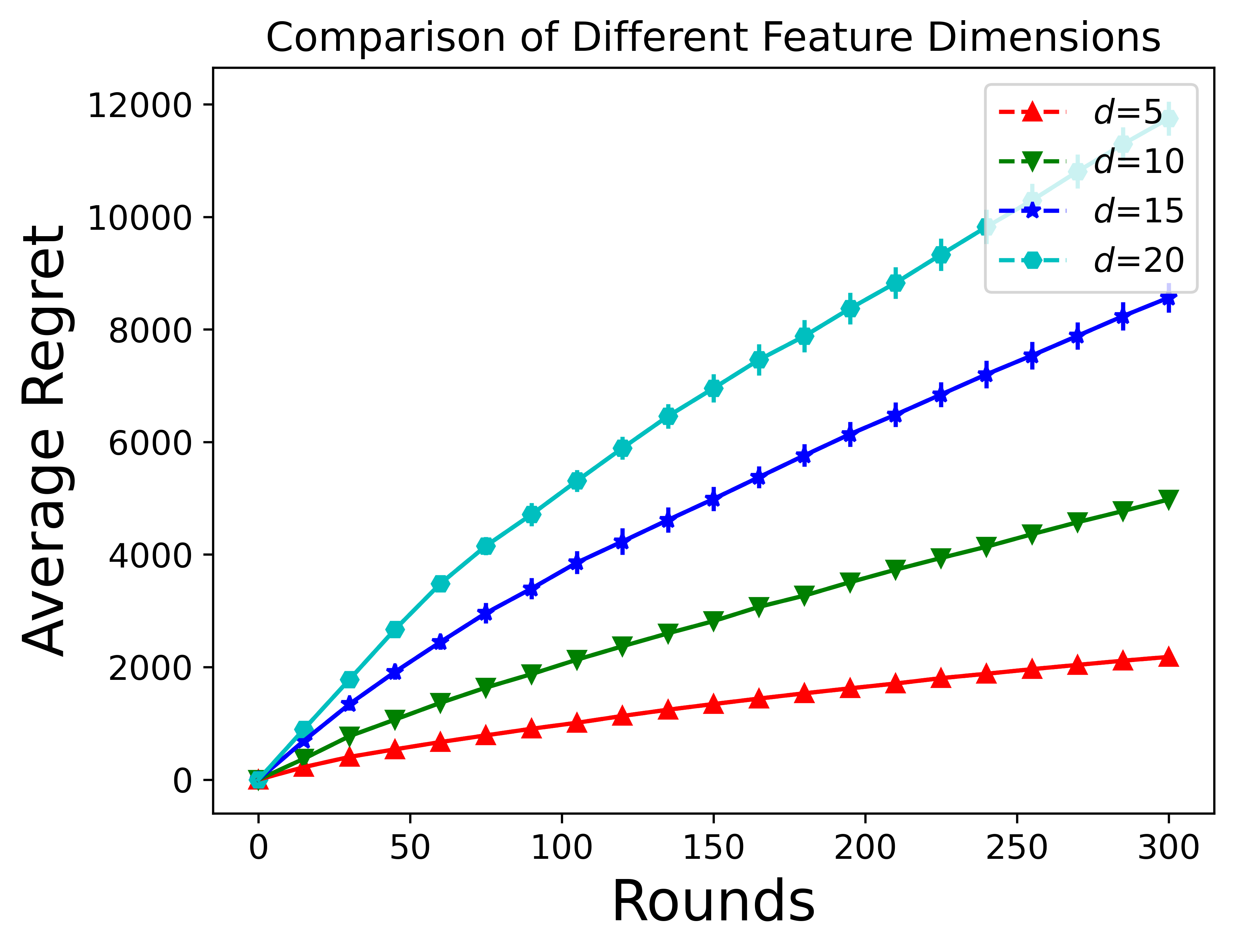}} \\
    \vspace{-1mm}
    \subfloat[Sub-Optimality Gap]{\label{fig:subg_arms}
		\includegraphics[width=0.27\linewidth]{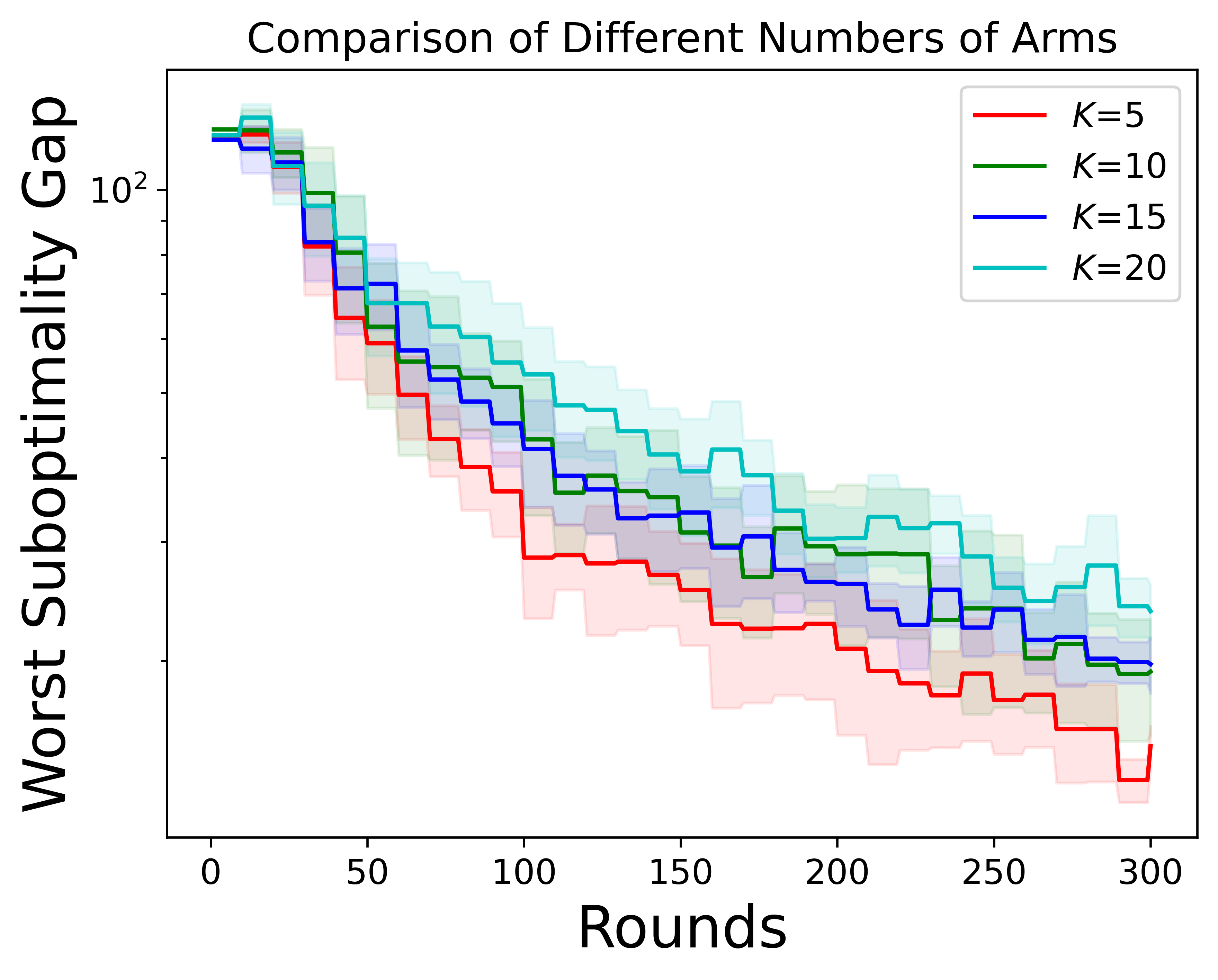}}
    \qquad
    \subfloat[MAE]{\label{fig:subg_avg_arms}
		\includegraphics[width=0.27\linewidth]{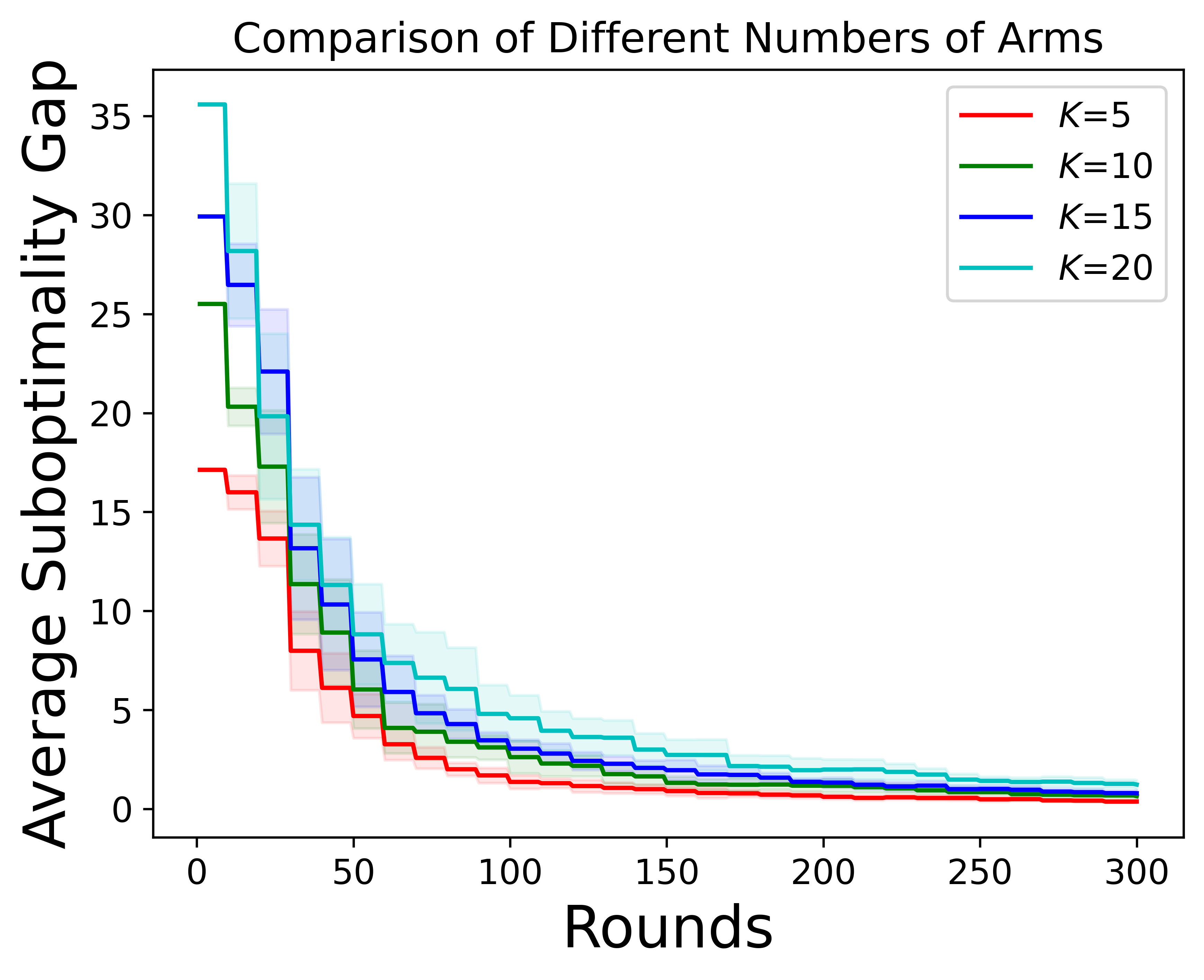}}
	\qquad
    \subfloat[Average Regret]{\label{fig:rmse_avg_arms}
		\includegraphics[width=0.27\linewidth]{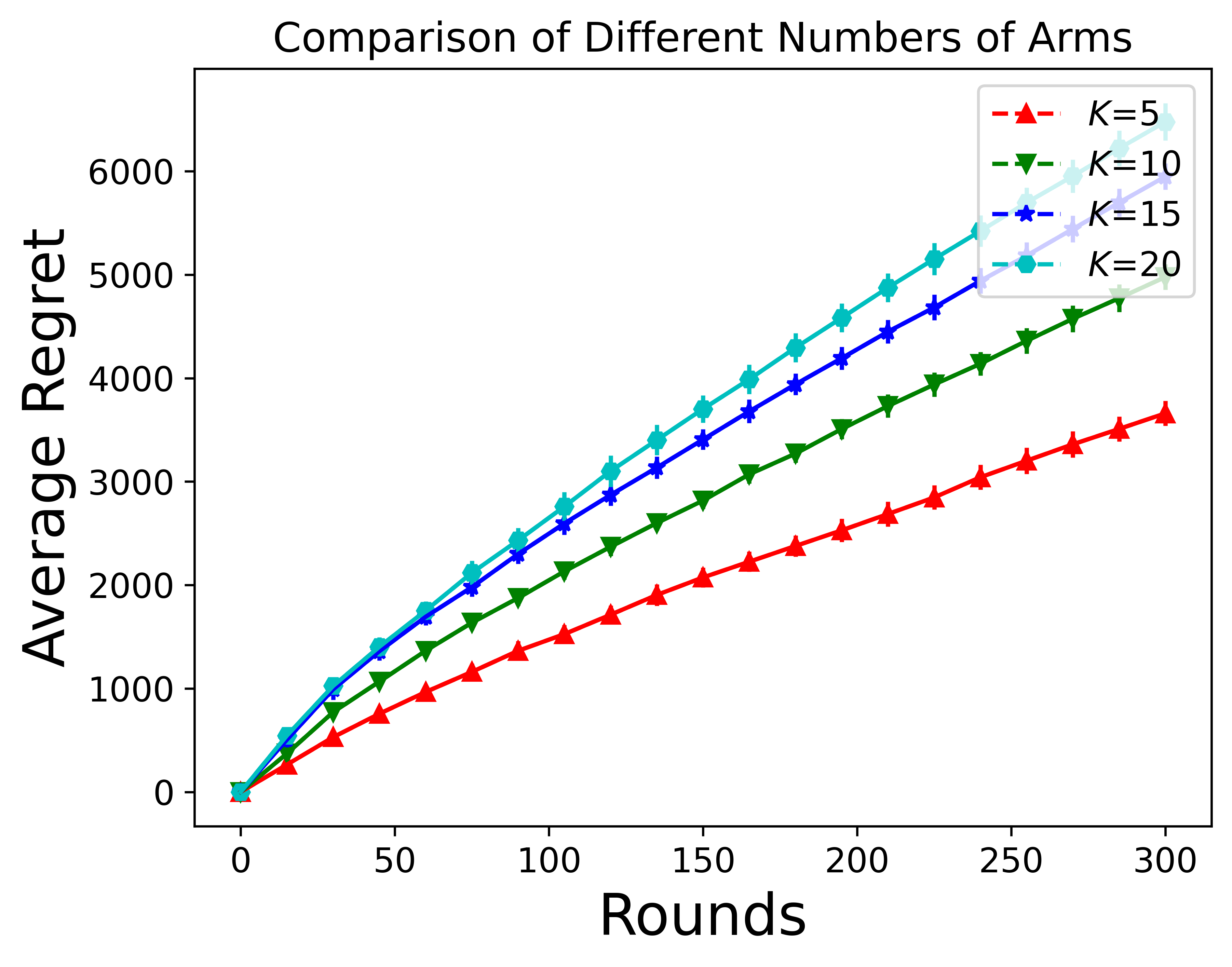}} 
	\caption{
        Performance of \algo{} (TS) on the Square function, evaluated across varying input dimensions (top row) and numbers of arms (bottom row).
	}
	\label{fig:neuraldb_ts_ablations}
    \vspace{-4mm}
\end{figure}

\para{Synthetic dataset.}
We generate sample features for each context-arm pair in a $d$-dimensional space. 
Let $x_{t, a}$ be the context-arm feature vector for context $c_t$ and an arm $a$. For all $t\ge 1$, $x_{t,a}$ is sampled uniformly at random from $(-1, 1)^d$.
We keep the number of arms constant across all rounds, denoted by $K$.
In our experiments, the binary preference feedback indicating whether $x_{t,1}$ preferred over $x_{t,2}$ (representing human preference feedback) is sampled from a Bernoulli distribution with parameter $\mu \Lp f(x_{t,1}) - f(x_{t,2})\Rp$), where $f$ is either a Square or Sine function..

\para{Reward function estimation.}
We use a neural network with $2$ hidden layers with width $50$ to estimate the latent reward function, $\lambda = 1.0$, $\delta=0.05$, $d = 20$, $K=10$, $T=1000$, and fixed value of $\nu_T = \nu = 1.0$ in all our experiments (unless we specifically indicate $d$ and $K$). Note that we did not perform any hyperparameter search for \algo{}, whose performance can be further improved by doing the hyperparameter search.

\para{Comparison with baselines.}
We compare the worst suboptimality gap (defined in \cref{eqn:suboptimality}), MAE (average suboptimality gap, i.e., $\sum_{t=1}^T\Lb\max_{a \in \cA} f(\phi(c_t,a)) - f\Lp\phi(c_t,\pi(c))\Rp \Rb/ T$), and average regret (defined in \cref{ssec:adb_regret}) against the different baselines of active contextual dueling bandits to evaluate the performance of UCB- and TS-variant of \algo{}. 
We use three baselines: Random, AE-Borda~\citep{arXiv23_mehta2023sample}, AE-DPO~\citep{arXiv23_mehta2023sample},  APO~\citep{arXiv24_das2024provably}, and the neural variants of AE-Borda and APO in which we use a neural network to estimate the latent reward function. They are named AE-Borda (NN) and \algo{} (APO) respectively. Experimental results in Fig.~\ref{fig:main-comparison} show that our algorithm, \algo{} (UCB), outperforms other baselines in almost all synthetic functions (i.e., square function and sine function) in terms of the suboptimality gap. 
Moreover, both UCB- and TS-variants of \algo{} also outperform other baselines on all synthetic functions in MAE and average regret. 
We have included more comparisons of our approach with other baselines in other settings (e.g., different $d$ or $K$) in \cref{sec:appendix_experiments}.

\para{Varying dimensions and arms vs. performance.}
As we increase the dimension of the context-arm feature vectors ($d$) and number of arms ($K$), the problem becomes more challenging. To assess how the changes in $K$ and$d$ affect the performance of our proposed algorithms, we vary $K = \{ 5, 10, 15, 20 \}$ and $d = \{5, 10, 15, 20 \}$ , while keeping the other parameters fixed. 
As expected, the performance of our algorithms gets worse with higher values of $K$ and $d$, as shown in Fig.~\ref{fig:neuraldb_ucb_ablations}. We have included similar results for \algo{} (TS) in \cref{fig:neuraldb_ts_ablations}.

%% file: latex/related_work.tex

In the following, we briefly review the relevant work, especially in neural contextual bandits and dueling bandits, to our problem setting.

\para{Neural Contextual Bandits.}
To model complex and non-linear reward functions, neural contextual bandits \citep{ICLR18_riquelme2018deep,zhou2020neural, zhang2020neural, ICLR22_xu2022neural, ArXiv25_bae2025neural, ICLR25_verma2025neural} use deep neural networks for reward function estimation. \citep{ICLR18_riquelme2018deep} employ multi-layer neural networks to learn arm embeddings and then use Thompson Sampling at the final layer for exploration. \cite{zhou2020neural} propose the first neural contextual bandit algorithm with sub-linear regret guarantees, using a UCB exploration strategy. Building on this, \cite{zhang2020neural} propose an algorithm with a TS exploration strategy. \cite{ICLR22_ban2022ee} introduces an adaptive exploration strategy incorporating an auxiliary neural network to estimate the potential gain of the exploitation neural network, diverging from traditional UCB and TS exploration strategies. To reduce the computational overhead of using gradient-based features, \cite{ICLR22_xu2022neural} only perform UCB-based exploration on the final layer of the neural network. More recent works \citep{ArXiv25_bae2025neural, ICLR25_verma2025neural} extend these techniques to handle neural contextual bandit settings with binary feedback (i.e., neural logistic bandits).

\para{Finite-Armed Dueling Bandits.}
Learning from preference feedback has been extensively studied in the bandit literature. In the finite-armed dueling bandits setting, the learner aims to find the best arm while only observing preference feedback for two selected arms \citep{ICML09_yue2009interactively,ICML11_yue2011beat,JCSS12_yue2012k}.  
To determine the best arm in dueling bandits, different criteria, such as the Borda winner, Condorcet winner, Copeland winner, or von Neumann winner, have been used while focusing on minimizing regret using only pairwise preference feedback \citep{ICML14_ailon2014reducing,WSDM14_zoghi2014relative,ICML14_zoghi2014relative,ICML15_gajane2015relative,COLT15_komiyama2015regret,UAI18_saha2018battle,AISTATS19_saha2019active,ALT19_saha2019pac,AISTATS19_verma2019online,ACML20_verma2020thompson,NeurIPS20_verma2020online,ICML23_zhu2023principled}. For a comprehensive overview of algorithms for various dueling bandits settings, we refer readers to the survey by \cite{JMLR21_bengs2021preference}.

\para{Contextual Dueling Bandits.}
Many real-life applications, such as online recommendations, content moderation, medical treatment design, prompt optimization, and aligning large language models, can be effectively modeled using contextual dueling bandits, where a learner observes a context (additional information before selecting a pair of arms) and then selects the arms based on that context and observes preference feedback for the selected arms. 
Since the number of context-arm pairs can be potentially large or even infinite, the mean latent reward of each context-arm is assumed to be parameterized by an unknown function of its features. Common assumptions include linear reward models \citep{NeurIPS21_saha2021optimal, ICML22_bengs2022stochastic, arXiv23_di2023variance, ALT22_saha2022efficient, arXiv24_li2024feelgood} and non-linear models \citep{ICLR25_verma2025neural}. 
For our setting, we adopt the neural contextual dueling bandit algorithms proposed in \citep{ICLR25_verma2025neural} to construct confidence ellipsoids for the latent non-linear reward function. Note that \algo{} can incorporate alternative confidence ellipsoids by appropriately modifying \cref{lem:confBounds}.
Furthermore, our work addresses an active learning problem and analyzes the convergence rate of the worst sub-optimality gap, whereas \citep{ICLR25_verma2025neural} focus on a regret minimization setting and derive upper bounds on cumulative regret.

\para{Active contextual dueling bandits.}
The work most closely related to ours is active contextual dueling bandit \citep{arXiv23_mehta2023sample,arXiv24_das2024provably}, which takes a principled approach to actively collecting preference datasets.
However, two key differences exist between our work and existing research: the non-linear reward function and the arm selection strategy. Existing studies typically assume a linear reward function, which may not be suitable for many real-world applications. 
Our work addresses this gap by extending the existing framework to incorporate non-linear reward functions in contextual dueling bandits. Additionally, existing approaches use different methods for selecting the pair of arms, leading to distinct arm selection strategies compared to ours. As a result of these differences in both the arm selection strategy and the non-linear reward function (which we estimate using a neural network), our analysis diverges significantly from that of prior work.

%% file: latex/conclusion.tex

This paper studies the problem of active human preference feedback collection by modeling it as an active neural contextual dueling bandit problem. We propose \algo{}, a principled and practical algorithm designed for efficiently gathering human preference feedback in scenarios where the reward function is non-linear. 
Exploiting the neural contextual dueling bandit framework, \algo{} extends its applicability to a broad range of real-world applications, including online recommendation systems and LLM alignment. Our theoretical analysis demonstrates that the worst suboptimality gap of \algo{} decays at a sub-linear rate as the preference dataset grows. 
Finally, our experimental results further validate these theoretical findings. An interesting direction for future work is applying \algo{} to real-life applications such as LLM alignment.
From a theoretical perspective, exploring the non-stationary setting presents another promising future direction.

%% file: latex/appendix.tex

To simplify the presentation, we use a common error probability of \( \delta \) for all probabilistic statements. Our final results naturally follow by applying a union bound over all individual $\delta$. 
Next, we will describe the key properties of positive definite matrices crucial for the subsequent proofs. These properties form the basis for several key parts of our analysis.

\begin{fact}[Properties of a positive definite matrix]
    \label{fact:pdMatrixProp}
    Let $V_0 = \lambda \bI_d$, $V_T =  V_0 + \sum_{s=1}^T z_s z_s^\top$ be a positive definite matrix, where $\lambda > 0$, $z_s \in \R^d$, and $\{Z_s = z_sz_s^\top\}_{s=1}^T$ is a finite adapted sequence of self-adjoint matrices, i.e., $V_s$ and $Z_s$ are $\cF_s$-measurable for all $s$, where $\cF_s$ represents all information available up to $s$.
    We use $\lambda_{\max}(V_T)$ and $\lambda_{\min}(V_T)$ to denote the maximum and minimum eigenvalue of matrix $V_T$. Then, the following properties hold for $V_T$: 
    \vspace{-2mm}
    \begin{enumerate}
        \item \label{prop:eigenValueCB} Let $\delta \in (0,1)$, $\forall s \le T:\ \norm{V_s - V_{s-1}}^2 \le C_s$, where $\norm{A}$ denotes the operator norm. Then, using Theorem 7.1 and Corollary 7.2 of \cite{FCM12_tropp2012user}, with probability at least $1-\delta$,
        \eqs{
            \Prob{\lambda_{\max}\Lp V_T - \EE{V_T}\Rp \ge \sqrt{8\sum_{s=1}^TC_s \log \Lp \frac{d}{\delta} \Rp}} \le \delta.
        }
        
        \item \label{prop:maxMinEigenValue} $\lambda_{\max}(V_T) = -{\lambda_{\min}(-V_T)}$.

        \item \label{prop:eigenValueRange} Let $\lambda_i(V)$ be the $i$-th eigenvalue of matrix $V$. If $W$ is any Hermitian matrix, then, from Weyl's inequality:
        \als{
            &1.\ \lambda_i(V_T) + \lambda_{\min}(W) \le \lambda_i(V_T + W) \le \lambda_i(V_T) + \lambda_{\max}(W) \text{ and }\\
            &2.\ \lambda_i(V_T) - \lambda_{\max}(W) \le \lambda_i(V_T - W) \le \lambda_i(V_T) - \lambda_{\min}(W).
        }

        \item \label{prop:maxInvMinEigenValue} Let $\forall z \in \R^d\colon\ \norm{z}_2 \le L$. Then, $\max_{z \in \R^d} \norm{z}_{V_T^{-1}} \le \norm{z}_2\sqrt{\lambda_{\max}(V_T^{-1})} \le L/\sqrt{\lambda_{\min}(V_T)}$.

        \item \label{prop:normConstMultiplier} For $a>0:\ \norm{az}_{V_T} = a\norm{z}_{V_T}$ and $\lambda_i(aV_T) = a\lambda_i(V_T)$.

    \end{enumerate}
    
\end{fact}

\subsection{Proof of \texorpdfstring{\cref{thm:normUB}}{Upper bound on Mahalanobis norm}}
We now prove the upper bound on the maximum Mahalanobis norm of a vector from the fixed input space, measured with respect to the inverse of a positive definite Gram matrix defined by finite, adapted samples from the same input space. 

\normUB*
\begin{proof}
    Using Property \ref{prop:eigenValueCB} in \cref{fact:pdMatrixProp} with $Y_T - \EE{Y_T} = \EE{V_T} - V_T$, we have
    \als{
        &\Prob{\lambda_{\max}\Lp \EE{V_T} - V_T\Rp \ge \tau} \le d\exp \Lp \frac{-\tau^2}{8\sum_{s=1}^T C_s} \Rp \\
        \implies &\Prob{-\lambda_{\min}\Lp -( \EE{V_T} - V_T)\Rp \ge \tau} \le d\exp \Lp \frac{-\tau^2}{8\sum_{s=1}^T C_s} \Rp &(\text{Property \ref{prop:maxMinEigenValue} in \cref{fact:pdMatrixProp}})\\
        \implies &\Prob{\lambda_{\min}\Lp V_T - \EE{V_T}\Rp \le -\tau} \le d\exp \Lp \frac{-\tau^2}{8\sum_{s=1}^T C_s} \Rp. \\
        \intertext{Using upper bound on $\lambda_{\min}\Lp V_T - \EE{V_T}\Rp$ from Property \ref{prop:eigenValueRange} in \cref{fact:pdMatrixProp}, we get}
        \implies &\Prob{\lambda_{\min}(V_T) - \lambda_{\min}(\EE{V_T}) \le -\tau} \le d\exp \Lp \frac{-\tau^2}{8\sum_{s=1}^T C_s} \Rp \\
        \implies &\Prob{\lambda_{\min}(V_T) \le \lambda_{\min}(\EE{V_T}) - \tau} \le d\exp \Lp \frac{-\tau^2}{8\sum_{s=1}^T C_s} \Rp.
    }
    Note that $\EE{V_T} = \EE{\sum_{t=1}^T z_sz_s^\top} = \sum_{t=1}^T \EE{z_sz_s^\top} = \sum_{t=1}^T \Sigma_s \le T\Sigma_{\max}$. Thus, we get
    \als{
        &\Prob{\lambda_{\min}(V_T) \le T\lambda_{\min}(\Sigma_{\max}) -\sqrt{8\sum_{s=1}^T C_s \log \Lp \frac{d}{\delta} \Rp}} \le \delta.
    }
    Therefore, with probability at least $1-\delta$, $\lambda_{\min}(V_T) \ge T\lambda_{\min}(\Sigma_{\max}) -\sqrt{8\sum_{s=1}^T C_s \log \Lp \nicefrac{d}{\delta} \Rp}$. Using Property \ref{prop:maxInvMinEigenValue} in \cref{fact:pdMatrixProp}, we now use to prove our key result as follows:
    \als{
        \max_{z \in \cZ}\norm{z}_{V_T^{-1}} &\le L/\sqrt{\lambda_{\min}(V_T)} \\
        &\le L/\sqrt{T\lambda_{\min}(\Sigma_{\max}) -\sqrt{8\sum_{s=1}^T C_s \log \Lp \frac{d}{\delta} \Rp} } \\
        &= L/G_T \\
        \implies \max_{z \in \cZ}\norm{z}_{V_T^{-1}} &\le = L/G_T. \qedhere
    }
\end{proof}

\subsection{Proof of \texorpdfstring{\cref{lem:subGapAbsUB} and \cref{lem:confBounds}}{Upper bound on worst sub-optimality gap}}
Our next results gives an upper bound of worst sub-optimality gap in terms of the upper bound of estimation error in the reward difference between any triple of context and two arms.
\subGapAbsUB*
\begin{proof}
    Define $a^\star = \argmax\limits_{a \in \cA} f(\phi(c,a))$. Recall the definition of worst suboptimality across all contexts, which is :
    \als{
        \Delta_{\cD_T}^{\pi} &= \max_{c \in \cC}\Lb\max_{a \in \cA} f(\phi(c,a)) - f\Lp\phi(c,\pi(c))\Rp \Rb \\
        &= \max_{c \in \cC}\Lb f(\phi(c,a^\star)) - f\Lp\phi(c,\pi(c))\Rp \Rb \\
        &= \max_{c \in \cC}\Lb f(\phi(c,a^\star)) - f\Lp\phi(c,\pi(c))\Rp + \hat{f}_T(\phi(c,a^\star)) - \hat{f}_T(\phi(c,a^\star)) \Rb \\
        &\le \max_{c \in \cC} \left|\Lb f(\phi(c,a^\star)) - f\Lp\phi(c,\pi(c))\Rp\Rb + \Lb\hat{f}_T(\phi(c,\pi(c))) - \hat{f}_T(\phi(c,a^\star)) \Rb \right| \\
        &=\max_{c \in \cC} \left|\Lb f(\phi(c,a^\star)) - f\Lp\phi(c,\pi(c))\Rp\Rb - \Lb\hat{f}_T(\phi(c,a^\star)) - \hat{f}_T(\phi(c,\pi(c)))\Rb \right| \\
        \implies \Delta_{\cD_T}^{\pi} &\le \max_{c \in \cC} \beta_T(c, a^\star, \pi(c)).
    }
    The inequality follows from the fact we have greedy policy, i.e., $\pi(c) = \argmin_{a\in \cA} \hat{f}_T(\phi(c, a))$ for any context $c$. Therefore, if $\pi(c) \ne a^\star$, then $\hat{f}_T(\phi(c, \pi(c))) \ge \hat{f}_T(\phi(c, a^\star))$ must hold.
\end{proof}

\confBounds*
\begin{proof}
    Recall that we are using the arm-selection strategies proposed in \citep{ICLR25_verma2025neural}. Since their confidence bounds hold for any adapted sequence of contexts, the proof of the first part follows directly from Theorem 1 in \citep{ICLR25_verma2025neural}, while the second part follows from Lemma 10 together with Eq. (27) of \citep{ICLR25_verma2025neural}.
\end{proof}

\begin{rem}   
    We adopt the arm selection strategies from the existing neural dueling bandit algorithms in \citep{ICLR25_verma2025neural}, which assume $\tilde{d} = o(T)$.
    In some cases, $\tilde{d} = \Omega(T)$~\citep{ICLR22_ban2022ee,ICLR24_deb2024contextual}, which may result in a constant convergence rate. 
    However, our objective is to demonstrate the use of the neural network for estimating non-linear reward functions in active contextual dueling bandits. Since neural dueling bandit algorithms primarily influence the arm selection strategy, we can incorporate any variants of these algorithms by making appropriate modifications to \cref{lem:confBounds}.
\end{rem}

\subsection{Proof of \texorpdfstring{\cref{thm:subGapUCB} and \cref{thm:subGapTS}}{Upper bound on worst sub-optimality gap}}

Equipped with \cref{thm:normUB}, \cref{lem:subGapAbsUB}, and \cref{lem:confBounds}, we will next prove the upper bound on the worst sub-optimality gap for a policy learned by \algo{} while using UCB- and TS-based arm selection strategy for a given context.

\subGapUCB*
\begin{proof}
    Using \cref{lem:confBounds} and setting value of $\beta_T(c, a^\star, \pi(c))$ using \cref{lem:confBounds} and \cref{eqn:sigma}, we have
    \als{
        \Delta_T^\pi &\le \max_{c \in \cC} \beta_T(c, a^\star, \pi(c)) &(\text{from \cref{lem:subGapAbsUB}}) \\
        &\le  \max_{c \in \cC} \Lp \nu_T \sigma_{T}(c, a^\star, \pi(c)) + 2\varepsilon'_{w,T} \Rp. &(\text{from \cref{lem:confBounds}}) \\
        \intertext{As $\nu_T$ and $\varepsilon'_{w,T}$ independent of context $c$, we get}
        \Delta_T^\pi &\le \nu_T \max_{c \in \cC} \Lp \sigma_{T}(c, a^\star, \pi(c))\Rp + 2\varepsilon'_{w,T} \\
        &= \nu_T\max_{c \in \cC} \Lp \sqrt{\frac{\lambda}{\kappa_\mu}} \norm{\frac{\phi(c,a^\star) - \phi(c,\pi(c))}{\sqrt{w}}}_{V_T^{-1}} \Rp + 2\varepsilon'_{w,T} &(\text{using \cref{eqn:sigma}}) \\
        &= \nu_T\max_{c \in \cC} \Lp \sqrt{\frac{\lambda}{\kappa_\mu w}} \norm{\phi(c,a^\star) - \phi(c,\pi(c))}_{V_T^{-1}} \Rp + 2\varepsilon'_{w,T} &(\text{Property \ref{prop:normConstMultiplier} in \cref{fact:pdMatrixProp}}) \\
        &= \nu_T\sqrt{\frac{\lambda}{\kappa_\mu w}} \max_{c \in \cC} \Lp \norm{\phi(c,a^\star) - \phi(c,\pi(c))}_{V_T^{-1}} \Rp + 2\varepsilon'_{w,T} \\
        &\le \nu_T\sqrt{\frac{\lambda}{\kappa_\mu w}} \Lp \frac{L}{\sqrt{T\lambda_{\min}(\Sigma_{\max}) -\sqrt{8\sum_{s=1}^T C_s \log \Lp \frac{d}{\delta} \Rp} }} \Rp + 2\varepsilon'_{w,T} &(\text{using \cref{thm:normUB}}) \\
        &\le  \tilde{O} \Lp \sqrt{\frac{\tilde{d}}{T}}\Rp. & \qedhere
    }
\end{proof}

\subGapTS*
\begin{proof}
    Using \cref{lem:confBounds} and setting value of $\beta_T(c, a^\star, \pi(c))$ using \cref{lem:confBounds} and \cref{eqn:sigma}, we have
    \als{
        \Delta_T^\pi &\le \max_{c \in \cC} \beta_T(c, a^\star, \pi(c)) \hspace{75mm}  (\text{from \cref{lem:subGapAbsUB}}) \\
        &\le  \max_{c \in \cC} \Lp \nu_T\log\Lp KT^2\Rp \sigma_{T}(c, a^\star, \pi(c)) + 2\varepsilon'_{w,T} \Rp. \hspace{36.5mm} (\text{from \cref{lem:confBounds}}) \\
        \intertext{The value of $\nu_T$ and $\varepsilon'_{w,T}$ independent of context $c$. By following similar steps to those in the proof of \cref{thm:subGapUCB}, we have}
        \Delta_T^\pi &\le \nu_T\log\Lp KT^2\Rp \max_{c \in \cC} \Lp \sigma_{T}(c, a^\star, \pi(c))\Rp + 2\varepsilon'_{w,T} \\
        &= \nu_T\log\Lp KT^2\Rp\max_{c \in \cC} \Lp \sqrt{\frac{\lambda}{\kappa_\mu}} \norm{\frac{\phi(c,a^\star) - \phi(c,\pi(c))}{\sqrt{w}}}_{V_T^{-1}} \Rp + 2\varepsilon'_{w,T} \\
        &= \nu_T\log\Lp KT^2\Rp\max_{c \in \cC} \Lp \sqrt{\frac{\lambda}{\kappa_\mu w}} \norm{\phi(c,a^\star) - \phi(c,\pi(c))}_{V_T^{-1}} \Rp + 2\varepsilon'_{w,T} \\ 
        &= \nu_T\log\Lp KT^2\Rp\sqrt{\frac{\lambda}{\kappa_\mu w}} \max_{c \in \cC} \Lp \norm{\phi(c,a^\star) - \phi(c,\pi(c))}_{V_T^{-1}} \Rp + 2\varepsilon'_{w,T} \\ 
        &\le \nu_T\log\Lp KT^2\Rp\sqrt{\frac{\lambda}{\kappa_\mu w}} \Lp \frac{L}{\sqrt{T\lambda_{\min}(\Sigma_{\max}) -\sqrt{8\sum_{s=1}^T C_s \log \Lp \frac{d}{\delta} \Rp} }} \Rp + 2\varepsilon'_{w,T} \\ 
        &\le  \tilde{O} \Lp \sqrt{\frac{\tilde{d}}{T}}\Rp. & \qedhere
    }
\end{proof}

%% file: latex/appendix_experiments.tex

\subsection{Experimental Details}
\para{Computational resources used for experiments.} 
All experiments were conducted on a server equipped with an AMD EPYC 7543 32-Core Processor, 256GB of RAM, and 8 NVIDIA GeForce RTX 3080 GPUs.

\para{Practical considerations.}
Based on the neural tangent kernel (NTK) theory \citep{NeurIPS18_jacot2018neural}, the initial gradient $g(x;\theta_0)$ can be used as serve as a surrogate for
the original feature vector $x$ as $g(x;\theta_0)$ effectively represents the random Fourier features of the NTK. 
To make our algorithm more practical, we use common practices in neural bandits \citep{zhou2020neural,zhang2020neural,ICLR25_verma2025neural}. 
Specifically, we replaced the theoretical regularization parameter $\frac{1}{2} w \lambda \norm{\theta - \theta_0}^2_{2}$ (where $w$ is the NN's width) with the simpler $\lambda \norm{\theta}^2_{2}$ in the loss function (defined in \cref{eqn:loss_function}) that is used to train our NN.
We retrain the neural network after every $20$ rounds for $50$ gradient steps across all experiments.

\subsection{Additional Experimental Results}
Next, we present the additional experiment results comparing the performance of \algo{} varying input dimension $d$ (\cref{fig:main-comp-dim}) and different numbers of arms $K$ (\cref{fig:main-comp-arms}). 

\begin{figure}[!ht]
	\vspace{-5mm}
    \centering
    \subfloat[Sub-Optimality Gap]{\label{fig:subg_20d}
		\includegraphics[width=0.27\linewidth]{figures/differentfunc/square_1000_10_20_1_compare7_1.0_1.0_300_10_10_subg.png}}
    \qquad
    \subfloat[MAE]{\label{fig:subg_avg_20d}
		\includegraphics[width=0.27\linewidth]{figures/differentfunc/square_1000_10_20_1_compare7_1.0_1.0_300_10_10_mae.png}}
	\qquad
    \subfloat[Average Regret]{\label{fig:rmse_avg_20d}
		\includegraphics[width=0.27\linewidth]{figures/differentfunc/square_1000_10_20_1_compare7_1.0_1.0_300_10_10_average_regret.png}} \\
    \vspace{-1mm}
    \subfloat[Sub-Optimality Gap]{\label{fig:subg_40d}
		\includegraphics[width=0.27\linewidth]{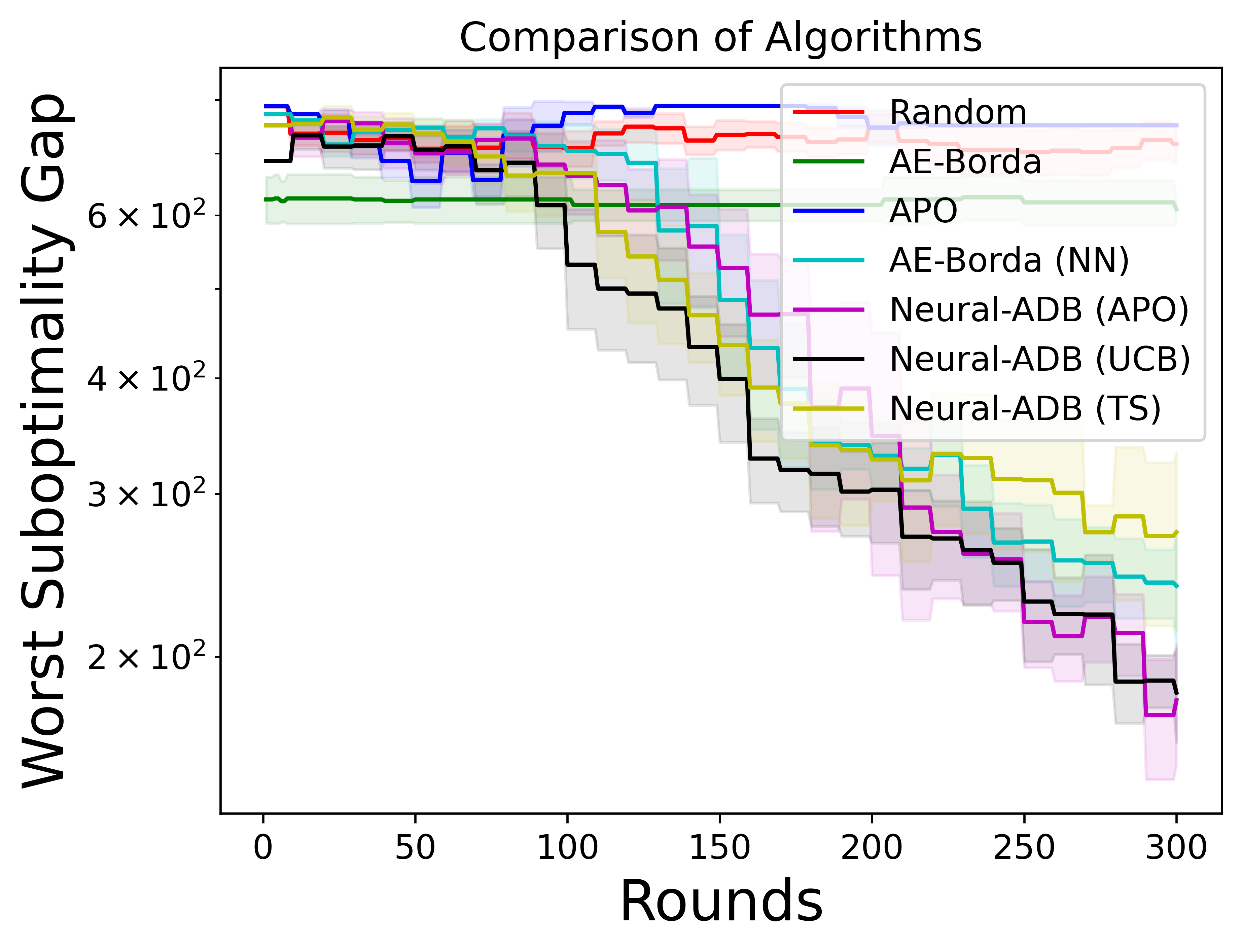}}
    \qquad
    \subfloat[MAE]{\label{fig:subg_avg_40d}
		\includegraphics[width=0.27\linewidth]{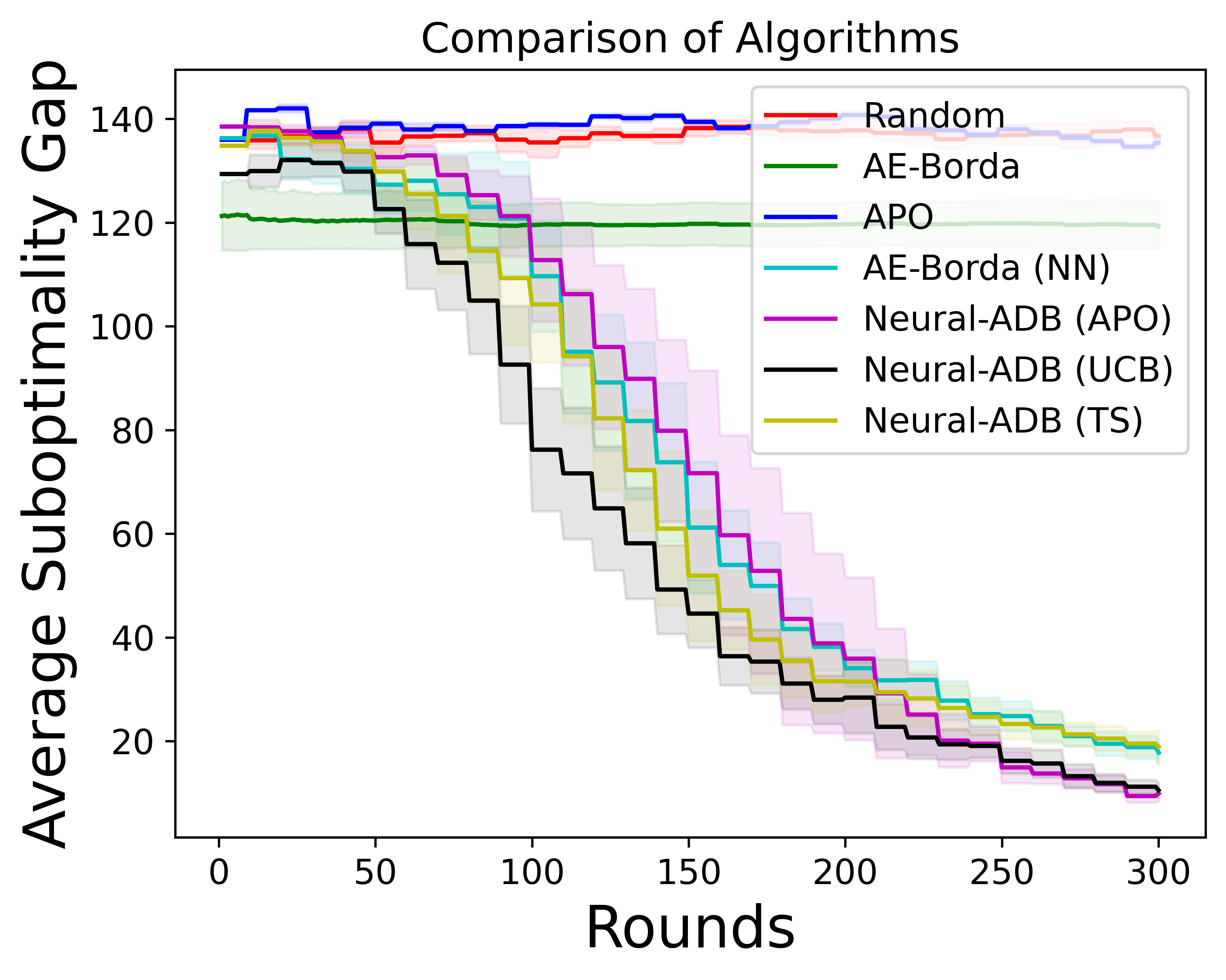}}
	\qquad
    \subfloat[Average Regret]{\label{fig:rmse_avg_40d}
		\includegraphics[width=0.27\linewidth]{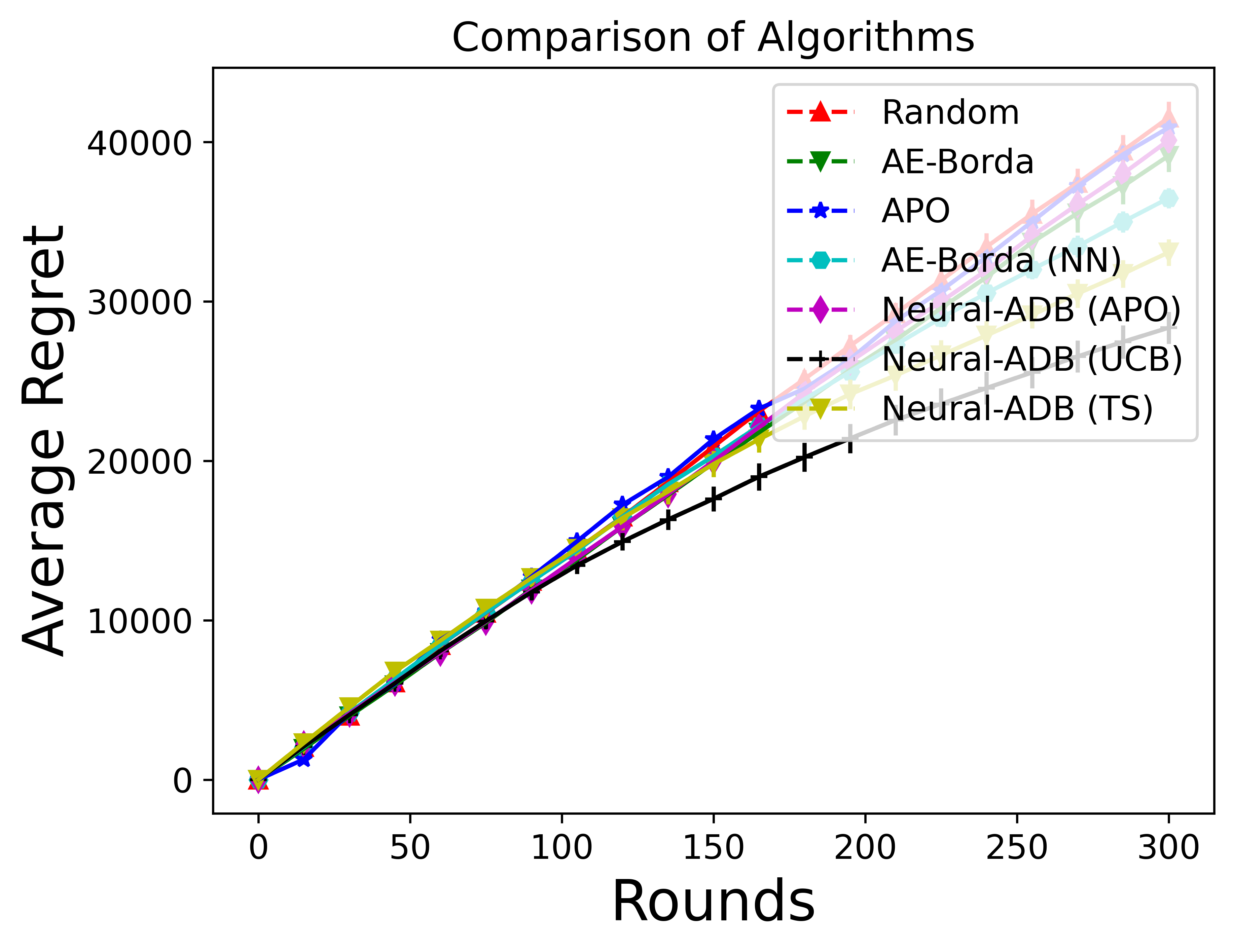}}
	\caption{
         Performance comparison across different input dimensions $d$: $d=20$ (first row) and $d=40$ (second row). We set the number of arms to $10$ and use the Square function for all experiments.
	}
	\label{fig:main-comp-dim}
\end{figure}

\begin{figure}[!ht]
	\centering
    \subfloat[Sub-Optimality Gap]{\label{fig:subg_5K}
		\includegraphics[width=0.27\linewidth]{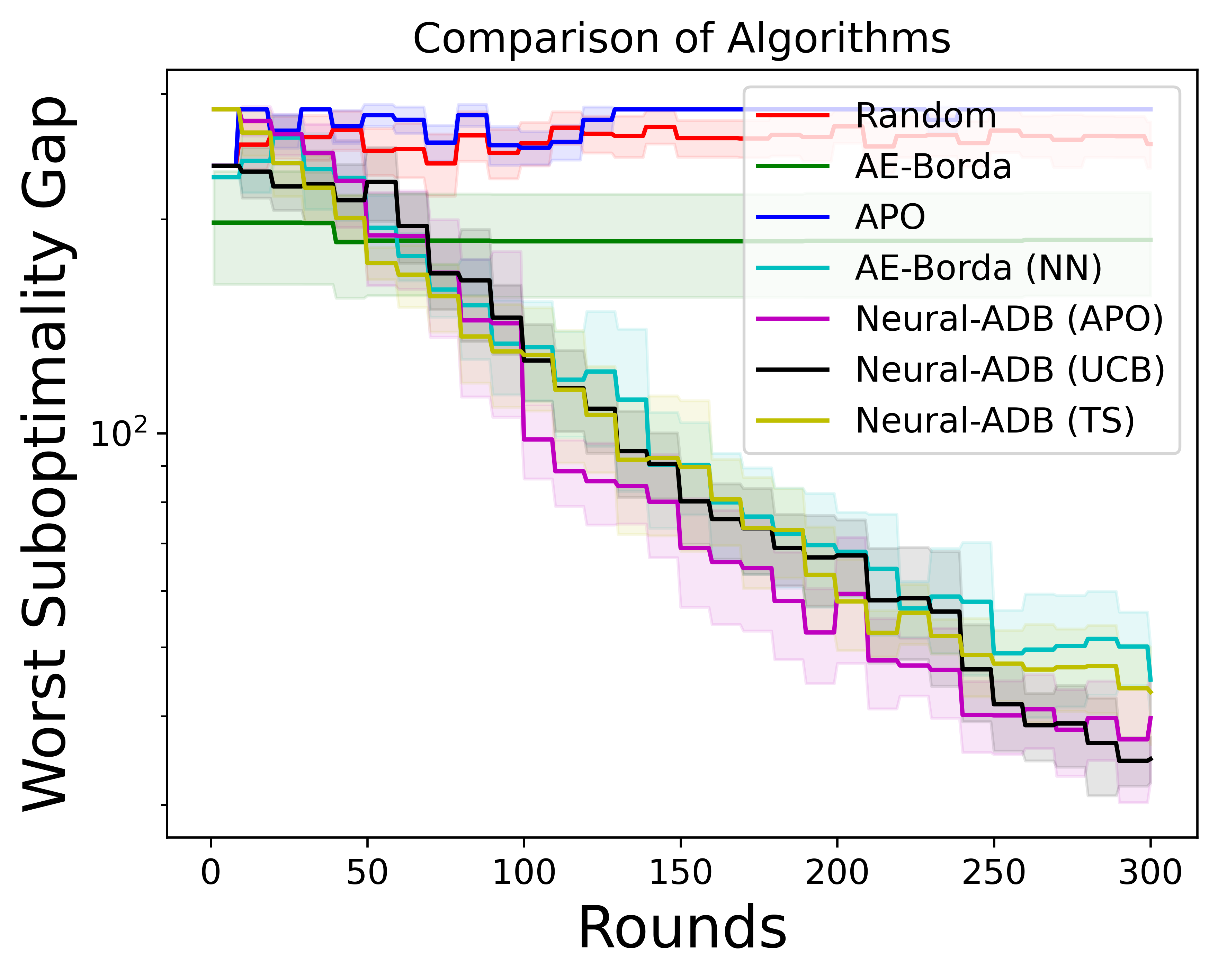}}
    \qquad
    \subfloat[MAE]{\label{fig:mae_avg_5K}
		\includegraphics[width=0.27\linewidth]{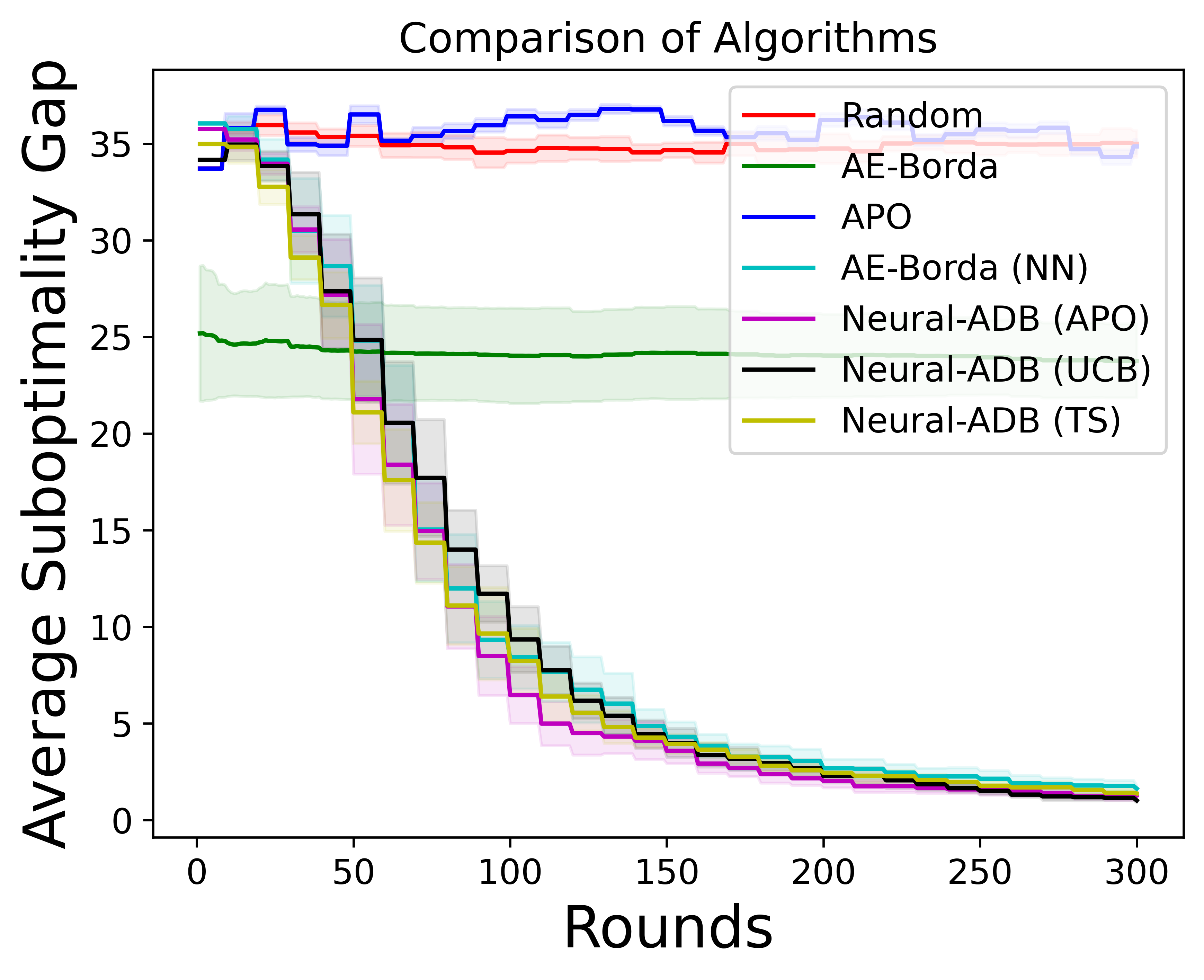}}
	\qquad
    \subfloat[Average Regret]{\label{fig:rmse_avg_5K}
		\includegraphics[width=0.27\linewidth]{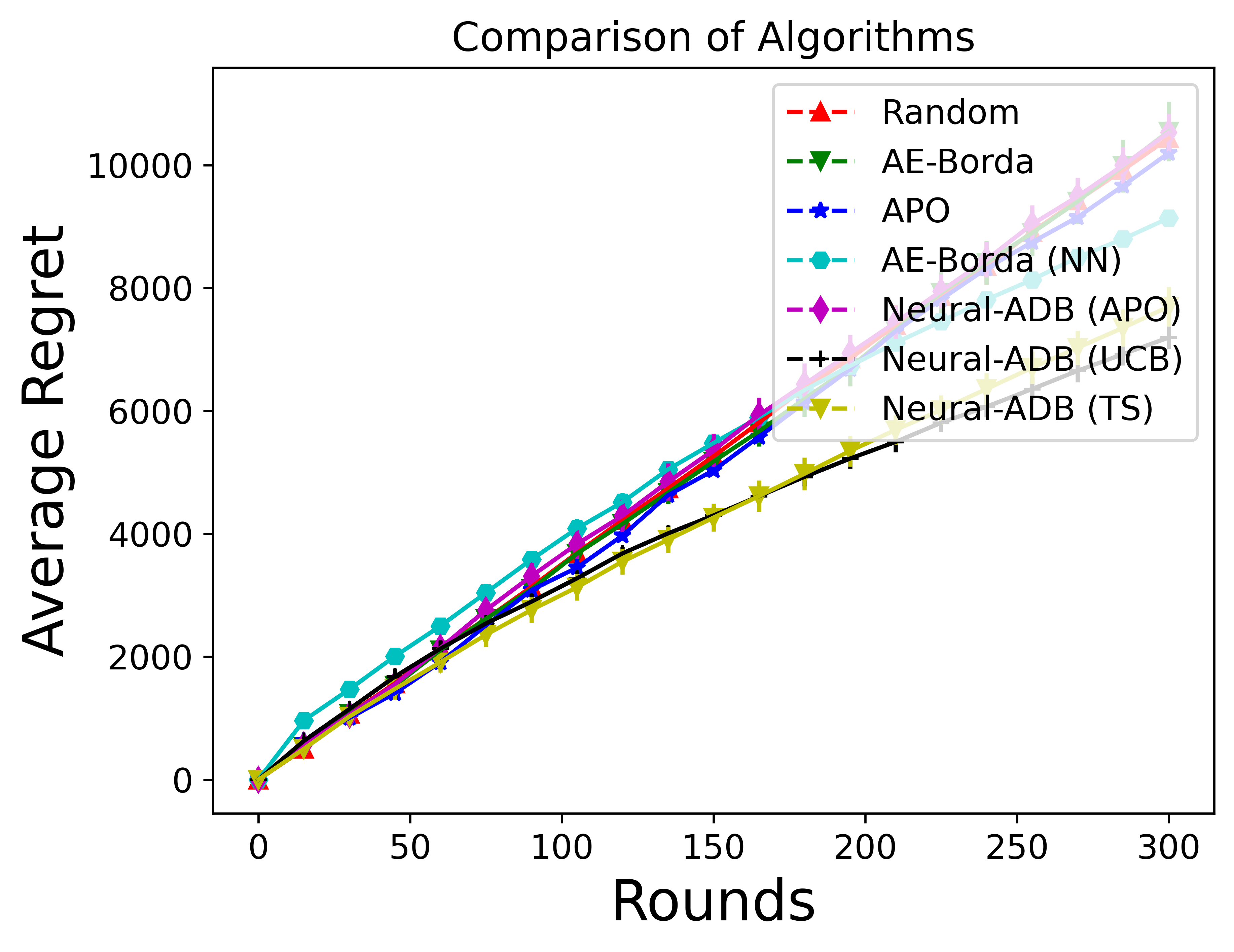}} \\

    \subfloat[Sub-Optimality Gap]{\label{fig:subg_10K}
		\includegraphics[width=0.27\linewidth]{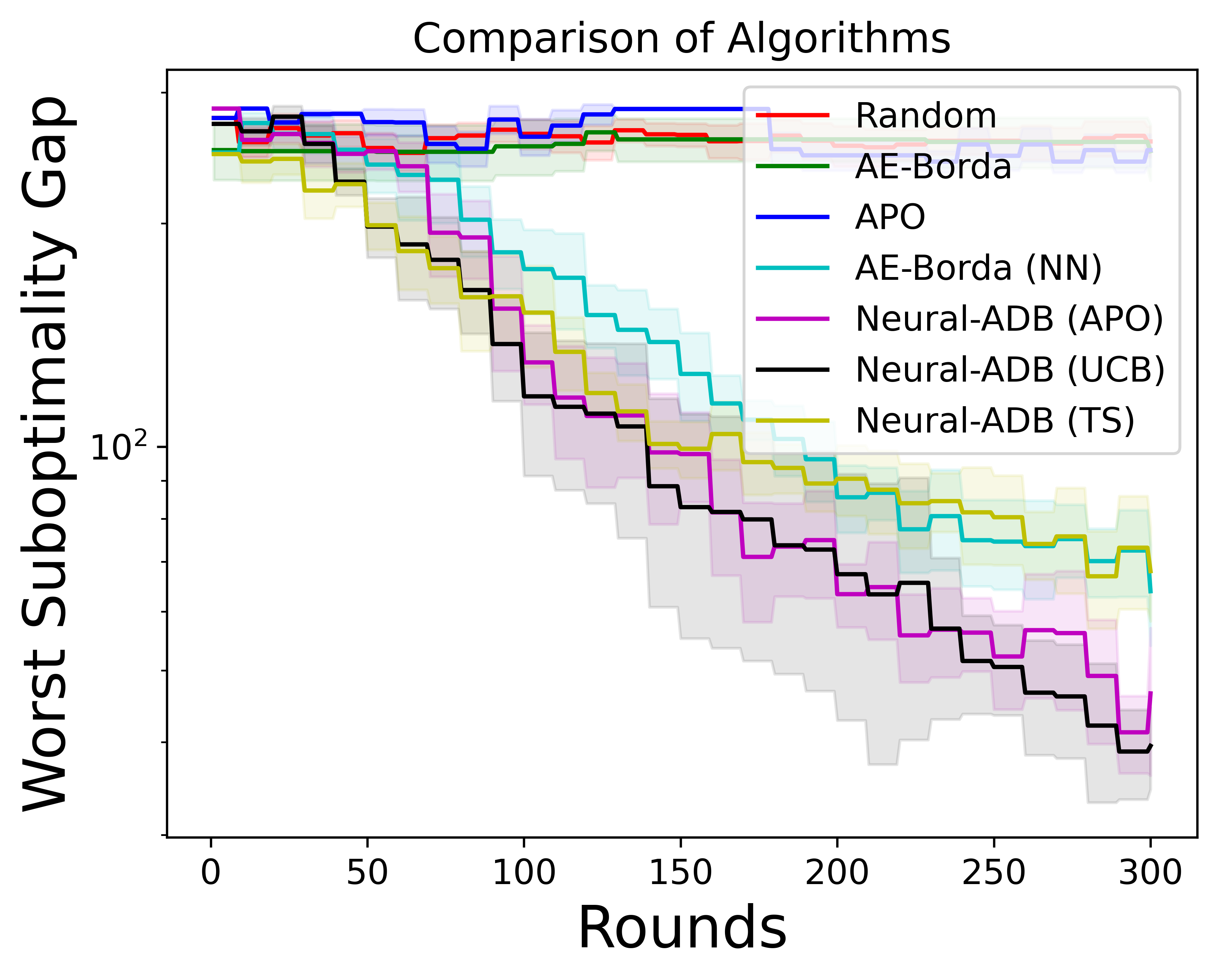}}
    \qquad
    \subfloat[MAE]{\label{fig:mae_avg_10K}
		\includegraphics[width=0.27\linewidth]{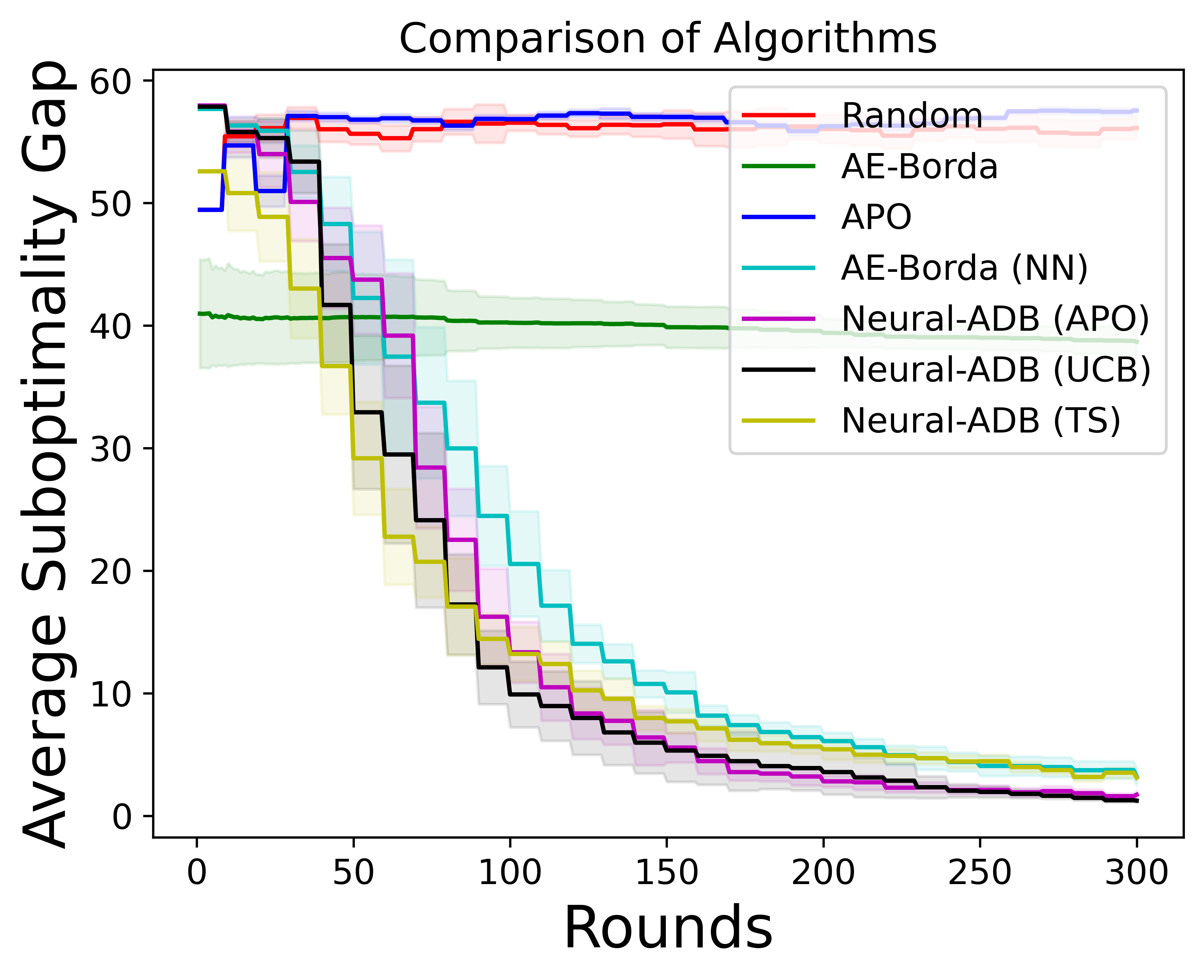}}
	\qquad
    \subfloat[Average Regret]{\label{fig:rmse_avg_10K}
		\includegraphics[width=0.27\linewidth]{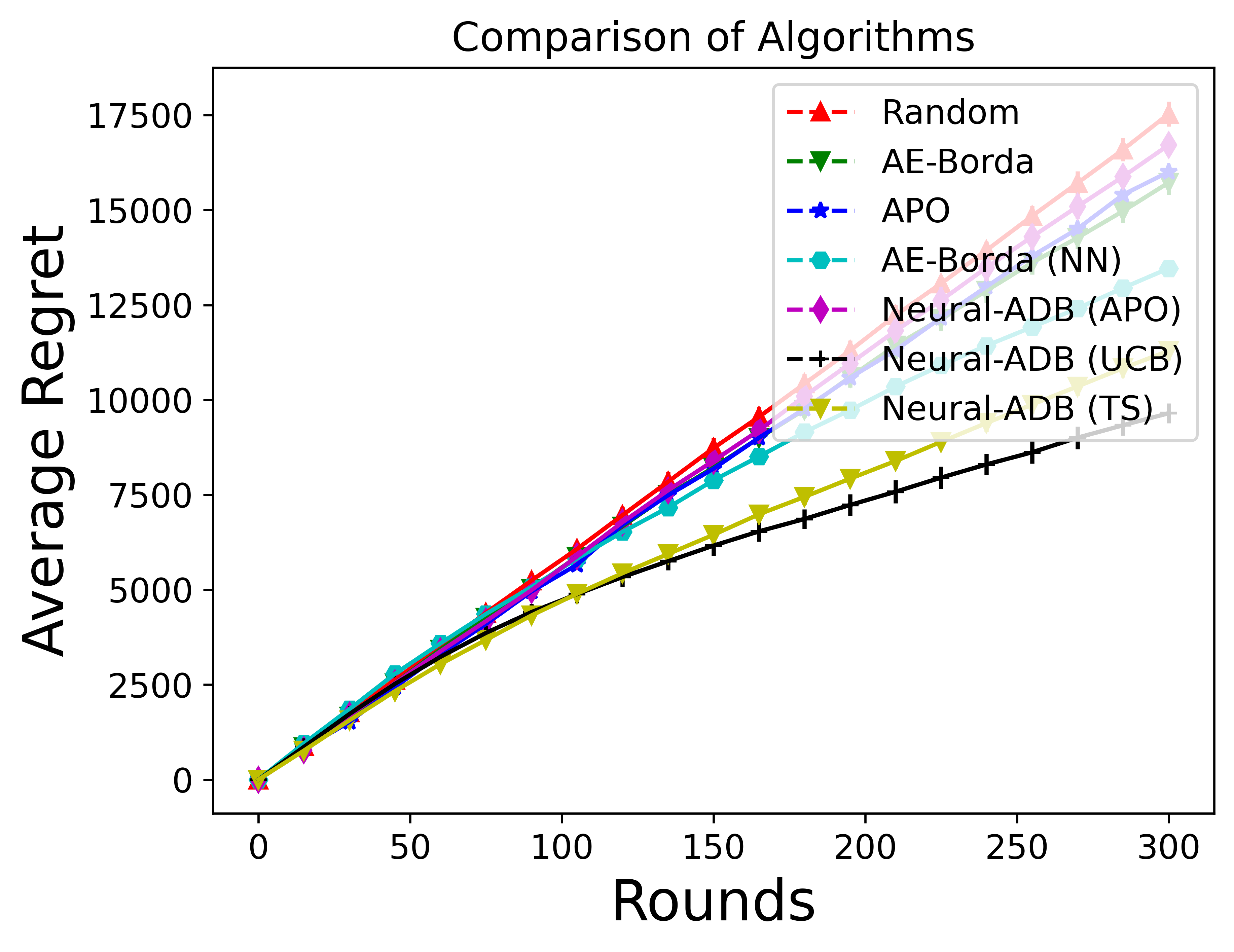}}  \\
        
    \subfloat[Sub-Optimality Gap]{\label{fig:subg_15K}
		\includegraphics[width=0.27\linewidth]{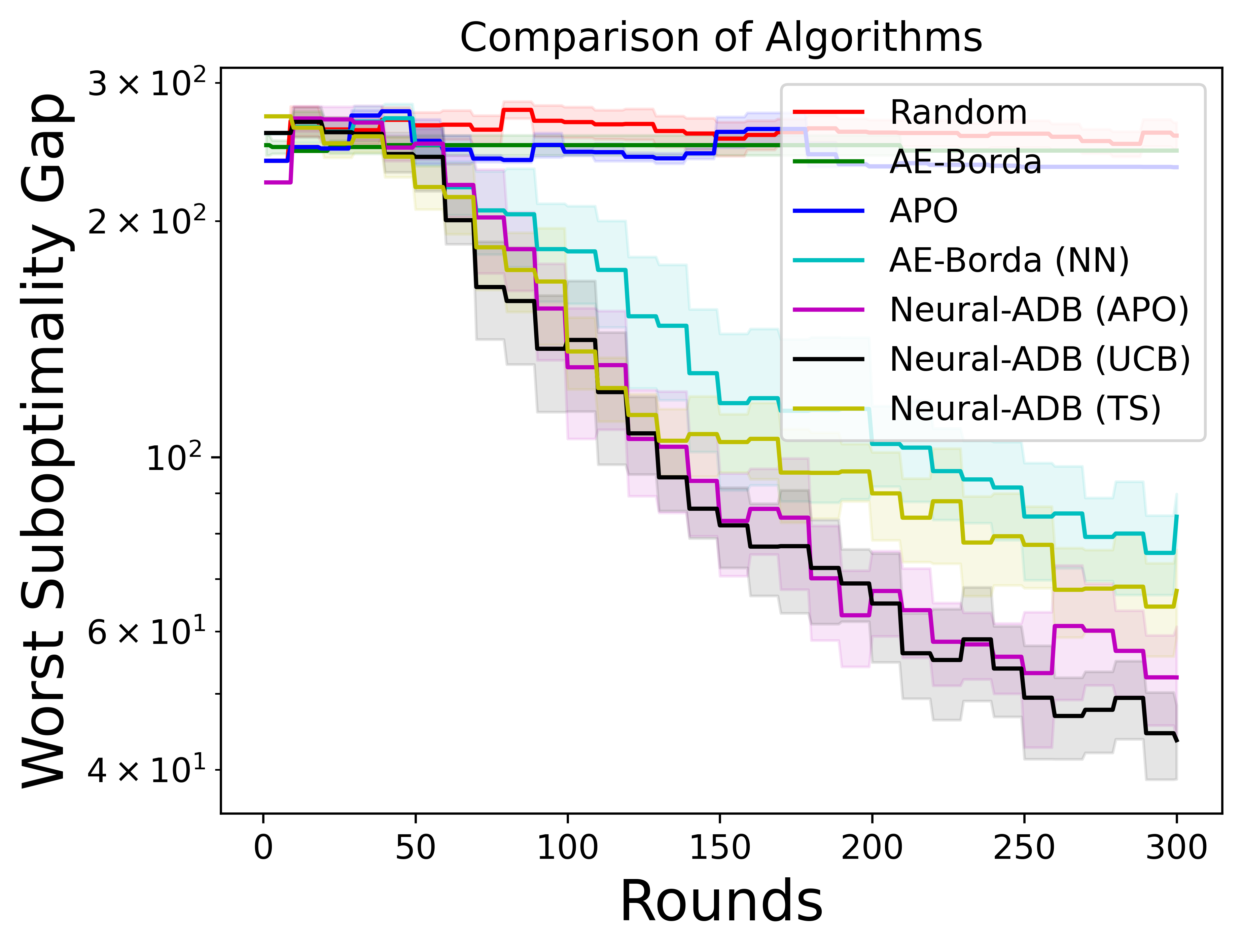}}
    \qquad
    \subfloat[MAE]{\label{fig:subg_avg_15K}
		\includegraphics[width=0.27\linewidth]{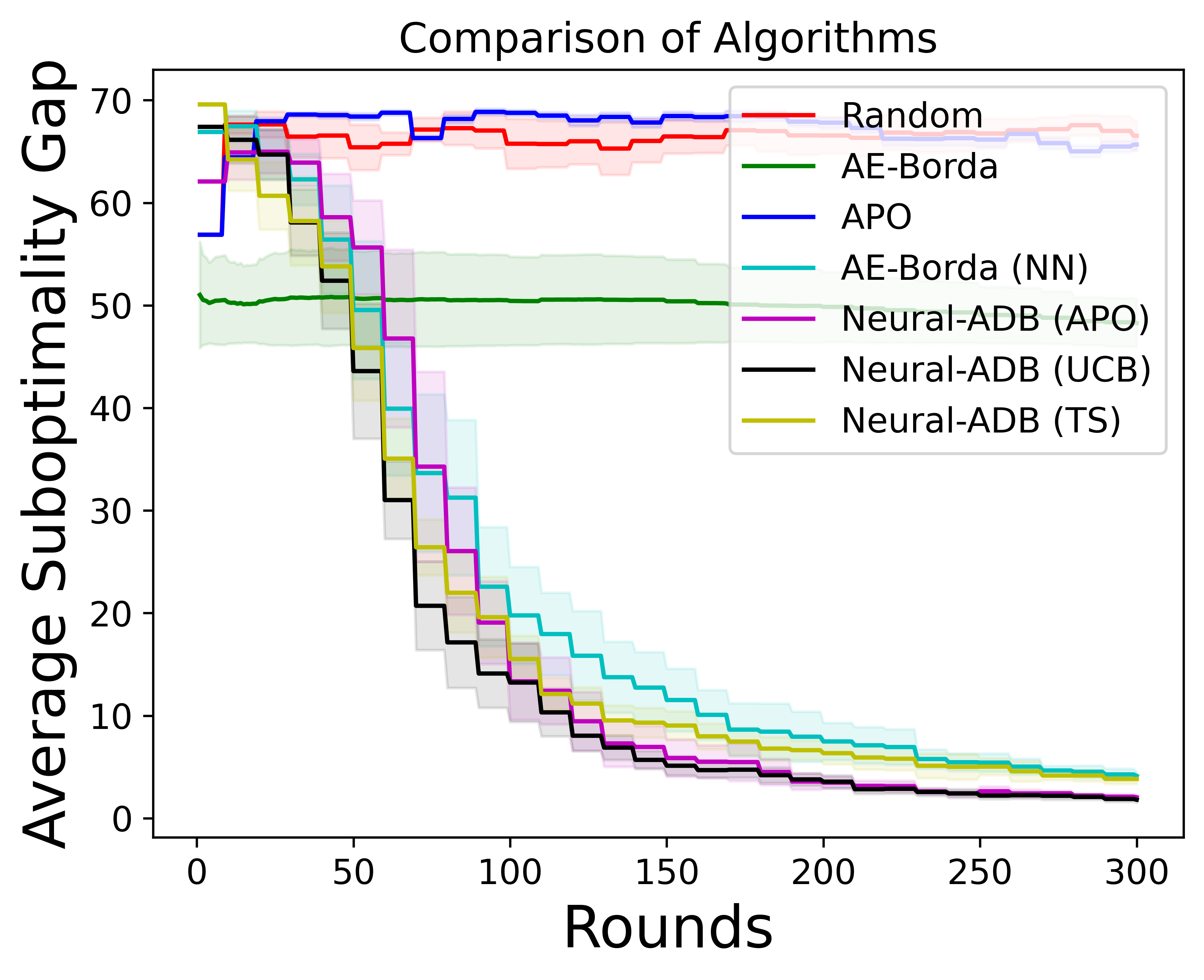}}
	\qquad
    \subfloat[Average Regret]{\label{fig:rmse_avg_15K}
		\includegraphics[width=0.27\linewidth]{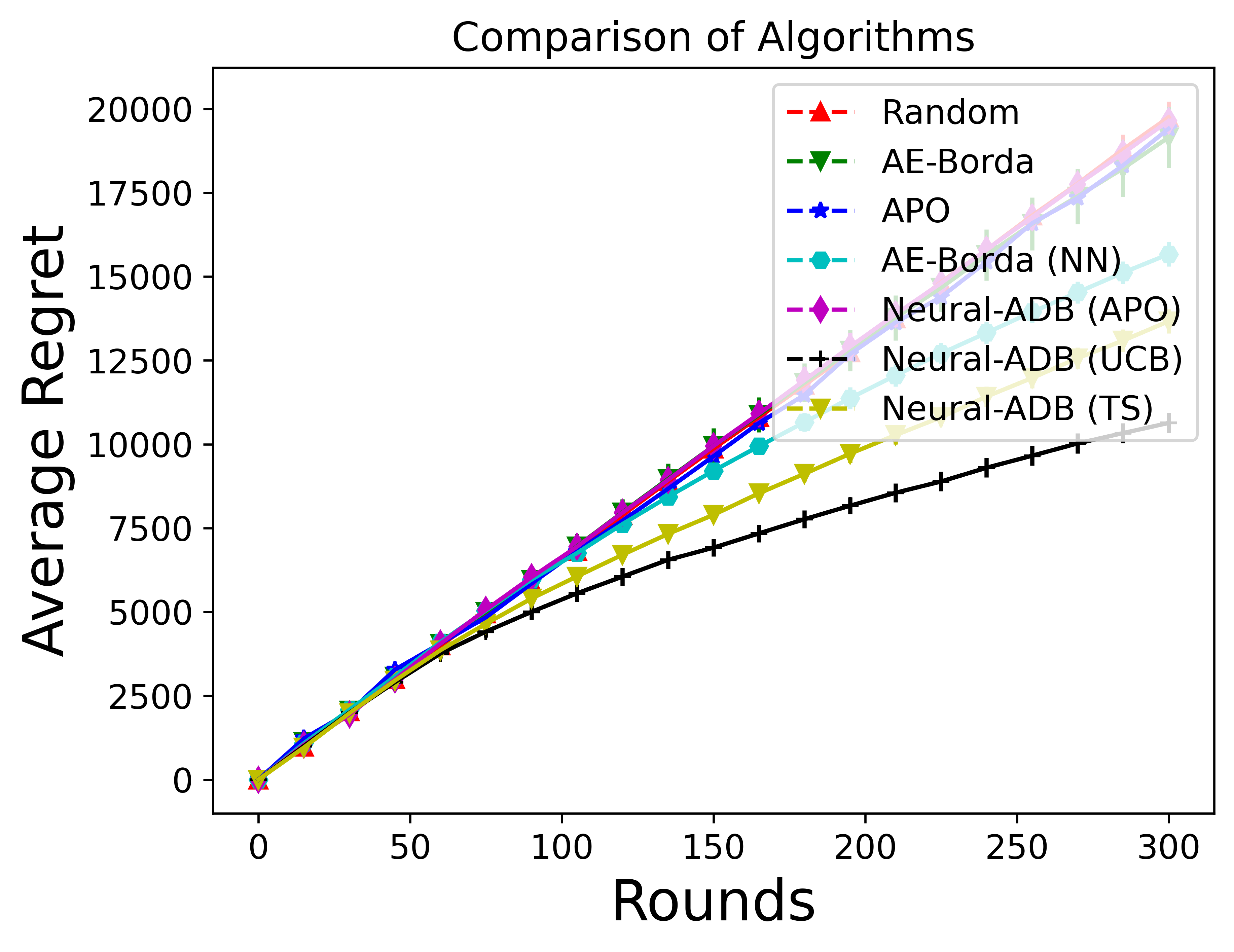}}  \\
	\caption{
        Performance comparison across different numbers of arms $K$: $K=5$ (top row), $K=10$ (middle row), and $K=15$ (bottom row). We set the input dimension to $20$ and use the Square function for all experiments.
	}
	\label{fig:main-comp-arms}
\end{figure}

\para{Performance vs. neural network size.} To investigate how performance varies with different neural network (NN) sizes, we used the Square and Cosine functions defined in the paper. We varied either the number of layers (with width = 32) or the width of the NN (with 2 layers), while keeping all other variables consistent with those in the paper. As shown in \cref{fig:nn_ablations}, we observed that selecting the appropriate size of NN is crucial for the given problem. Using a large NN for a simple problem leads to poor performance due to high bias in the estimation, while a smaller NN may not accurately be able to estimate the complex non-linear function.
\begin{figure}[!ht]
	\vspace{-5mm}
    \centering
    \subfloat[Sub-Optimality Gap]{\label{fig:subg_hidden}
		\includegraphics[width=0.27\linewidth]{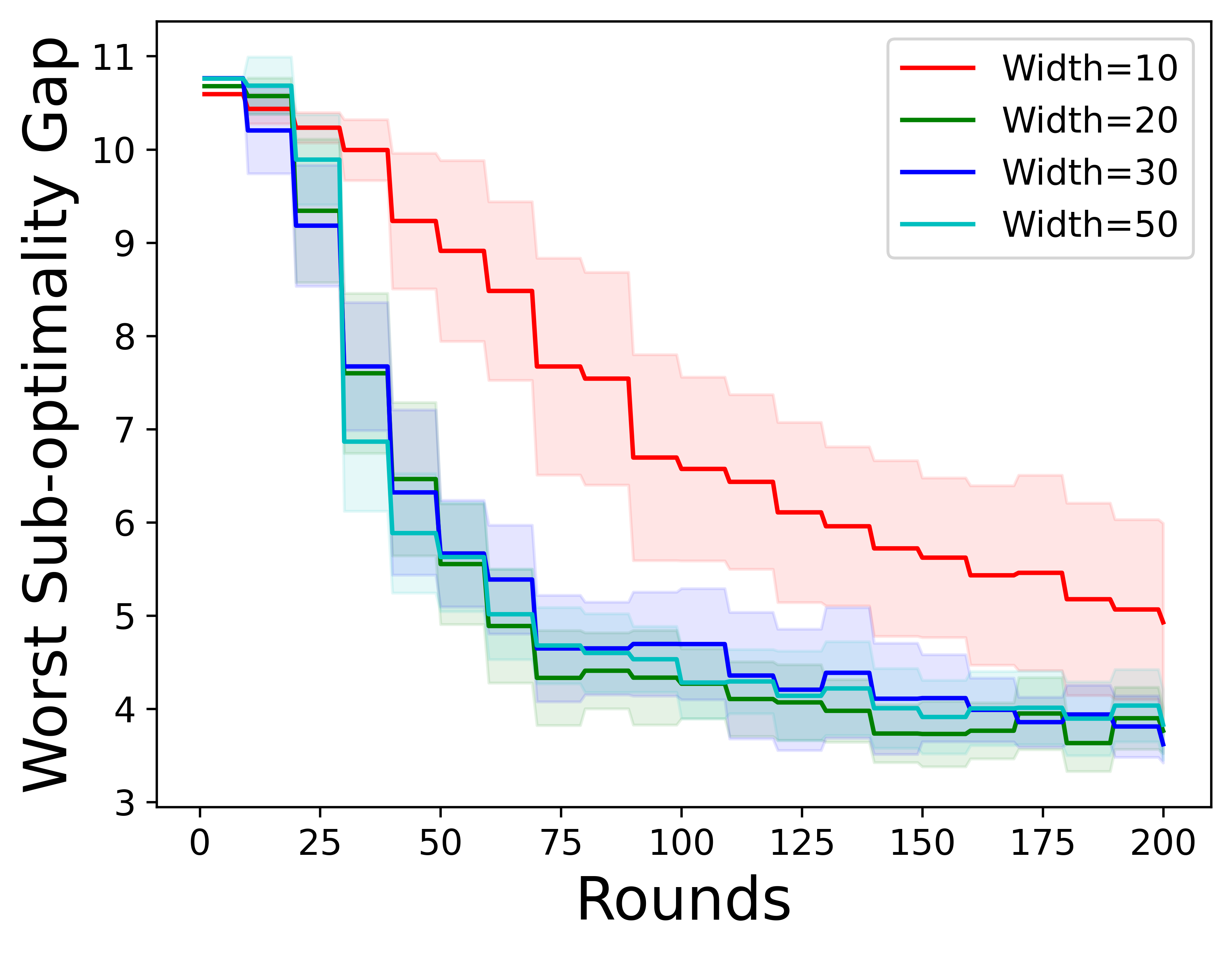}}
    \qquad
    \subfloat[MAE]{\label{fig:mae_hidden}
		\includegraphics[width=0.27\linewidth]{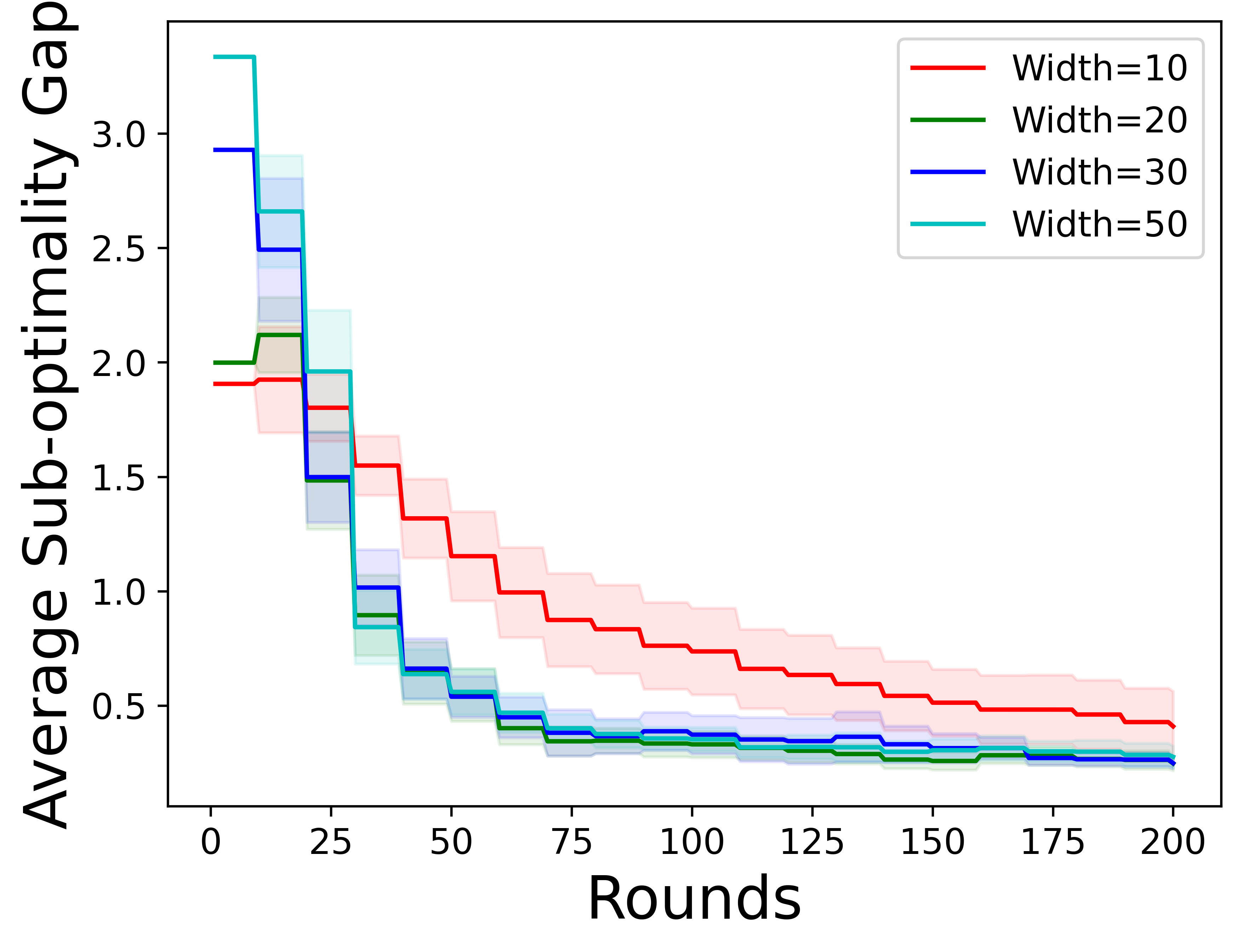}}
	\qquad
    \subfloat[Average Regret]{\label{fig:rmse_hidden}
		\includegraphics[width=0.27\linewidth]{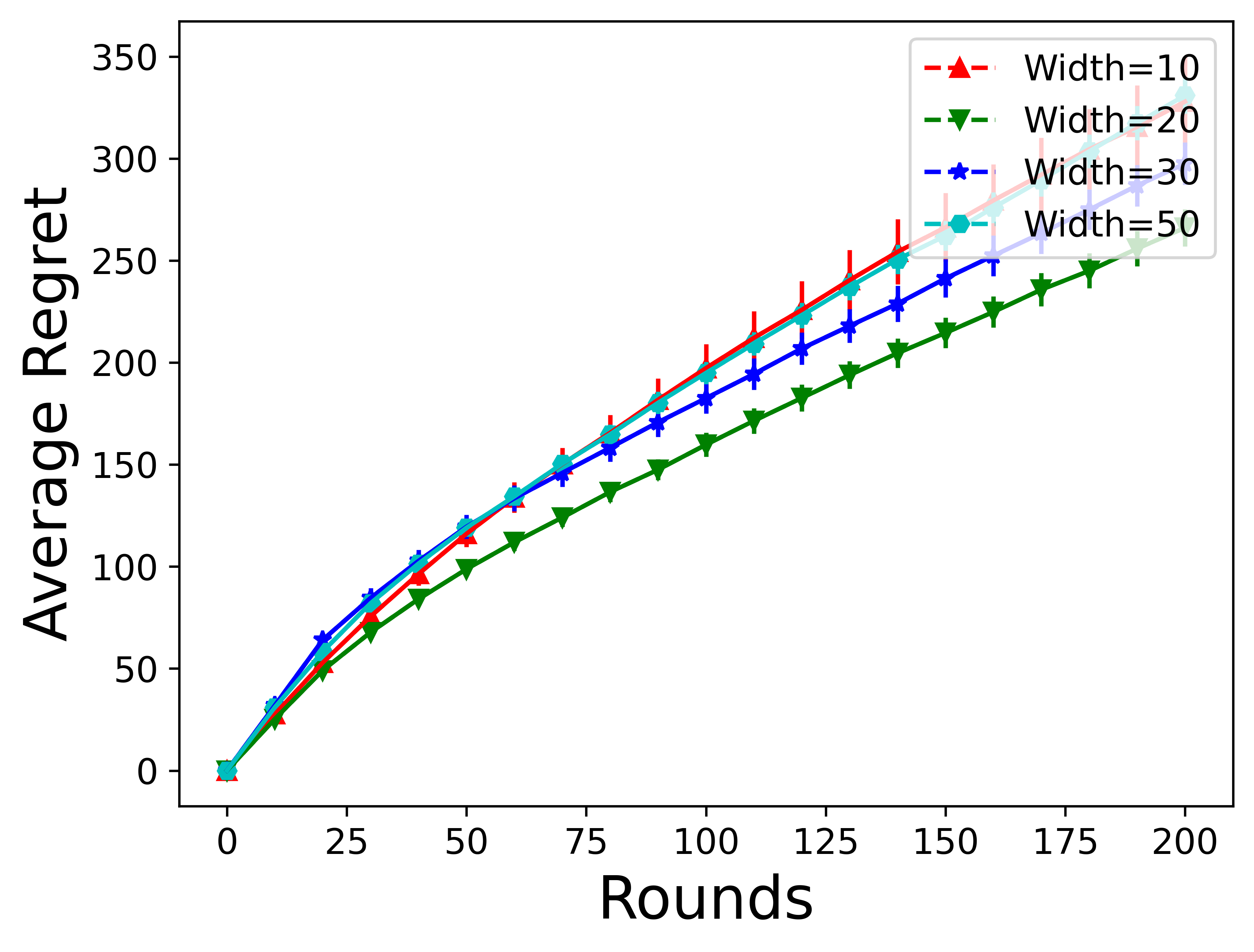}} \\
    \vspace{-1mm}
    \subfloat[Sub-Optimality Gap]{\label{fig:subg_layers}
		\includegraphics[width=0.27\linewidth]{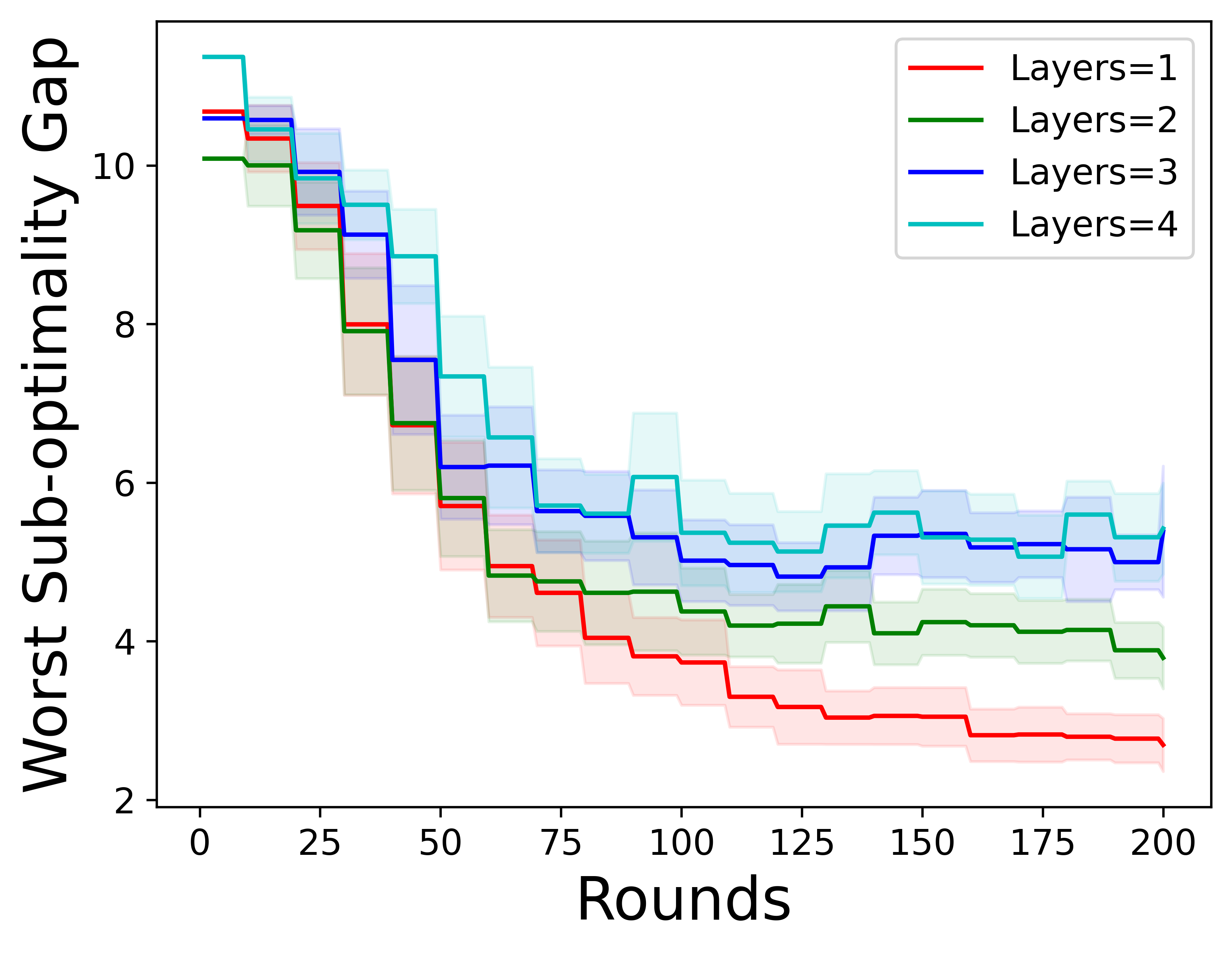}}
    \qquad
    \subfloat[MAE]{\label{fig:mae_layers}
		\includegraphics[width=0.27\linewidth]{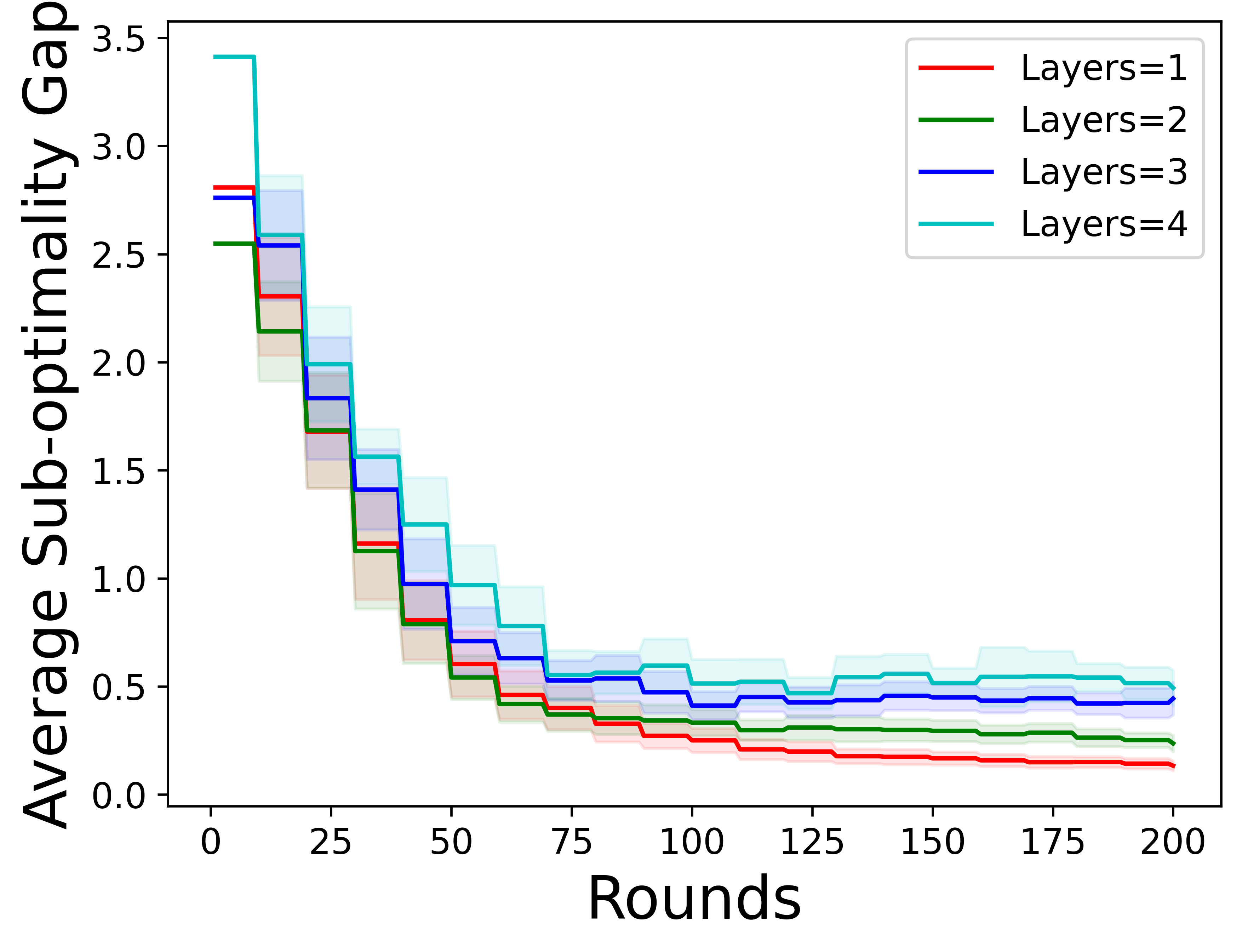}}
	\qquad
    \subfloat[Average Regret]{\label{fig:rmse_layers}
		\includegraphics[width=0.27\linewidth]{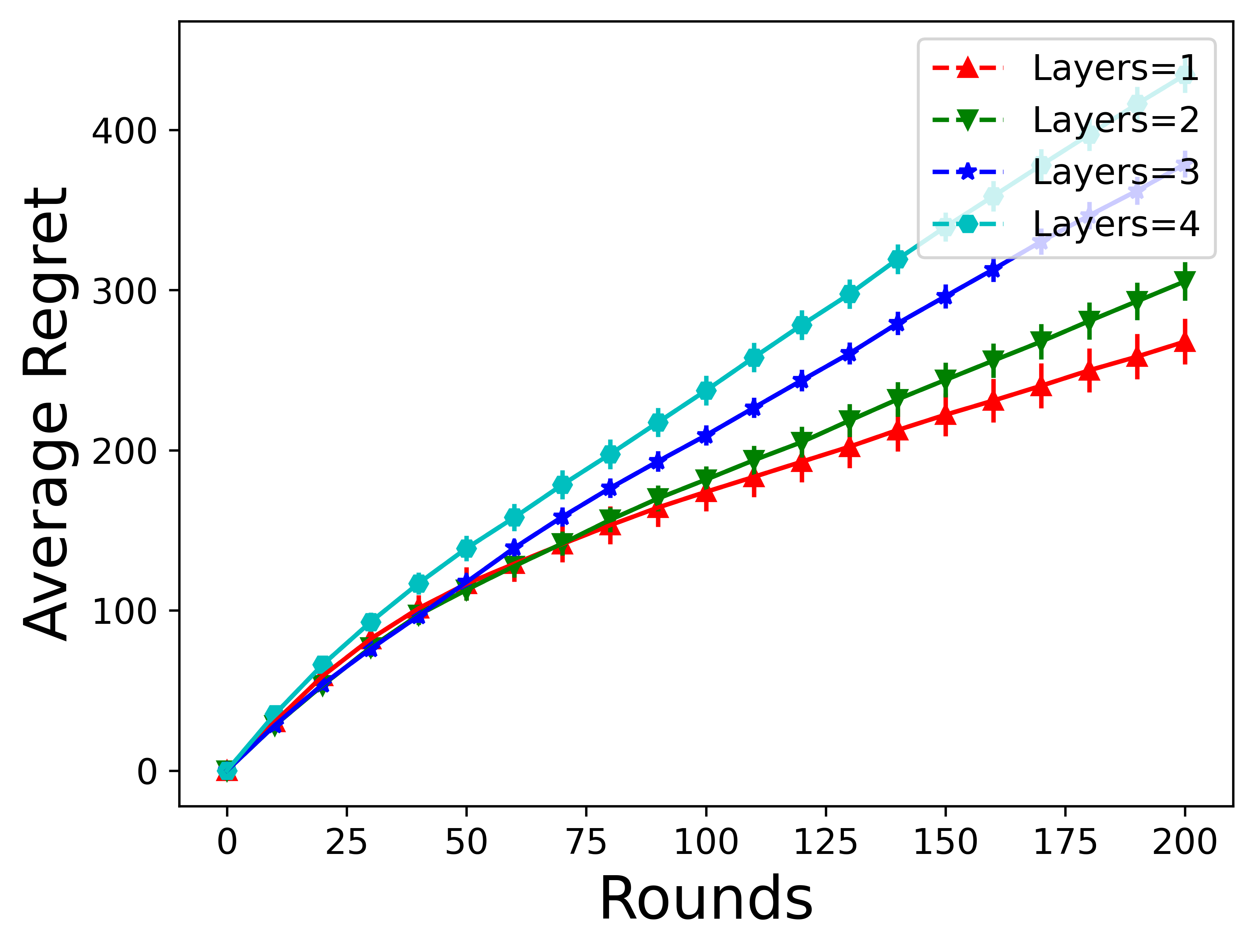}} \\
	\caption{
         We compare performance across different neural network widths (first row) and numbers of hidden layers (second row), using the Square function in all experiments. All other parameters are kept fixed, except that the width is set to $32$ when varying the number of layers.
	}
	\label{fig:nn_ablations}
\end{figure}

\subsection{Computational Efficiency.}
To discuss the computational efficiency of \algo{}, we follow the approach of \citep{ICLR25_verma2025neural} and consider the following two key aspects: size of the neural network and then the number of contexts and arms.

\para{Size of the neural network.} 
The primary computational cost in \algo{} arises from the neural network (NN) used to approximate the latent non-linear reward function. Given a context-arm feature vector of dimension $d$, an NN with $D$ hidden layers and $w$ neurons per layer incurs an inference cost of $O(dw + Dw^2 + w)$ per context-arm pair. The total number of parameters in the NN is $p = dw + Dw^2 + w$, and the training time per iteration is $O(\cE\cP Dw^2)$, where $\cE$ is the number of training epochs and $\cP$ is the number of observed context-arm pairs. Choosing an appropriate NN size is critical, as NNs that are too small may fail to accurately approximate the underlying non-linear reward function, while excessively large NNs can result in substantial training and inference overhead.

\para{Number of contexts and arms.}
Let $K$ denote the number of arms and $p$ the total number of NN parameters. Since \algo{} uses NN gradients as context-arm features, the cost of computing gradients for all arms per context is $O(K^2dp)$, where $d$ is the dimension of the context-arm feature vector.
The cost of computing reward estimates and confidence terms for all context-arm pairs is $O(K^2p)$ and $O(K^2p^2)$, respectively. For arm selection, the first selection step requires $O(Kp + K)$, consisting of reward estimation for all arms ($O(Kp)$) and then identifying the arm with the highest estimated reward ($O(K)$).
The second arm selection incurs a cost of $O(Kp + (K-1)p^2)$, including reward estimation $O(Kp)$ and confidence term computation $O((K-1)p^2)$ relative to the first selected arm. 
Thus, the total computational cost for selecting a pair of arms per context is $O(K^2dp + K^2p^2)$. Since each context-arm pair is independent, gradients, reward estimates, and optimistic terms can be computed in parallel, reducing the overall cost to $O(dp + p^2)$ for each iteration. \\